%% file: main.tex
\documentclass[11pt,a4paper]{article}

\usepackage[T1]{fontenc}
\usepackage[utf8]{inputenc}
\usepackage{CJKutf8}
\usepackage{times}

\usepackage[margin=2.4cm,headheight=38pt,headsep=18pt,footskip=28pt]{geometry}
\usepackage{setspace}
\usepackage[table,dvipsnames]{xcolor}
\usepackage{colortbl}
\usepackage{graphicx}
\usepackage{booktabs}
\usepackage{array}
\usepackage{multirow}
\usepackage{tabularx}

\usepackage{amsmath,amssymb}

\definecolor{AMapBlue}{HTML}{48B8F0}
\definecolor{AMapDeepBlue}{HTML}{0877C9}
\definecolor{AMapSky}{HTML}{EAF8FF}
\definecolor{AMapLine}{HTML}{9EDBFA}
\definecolor{TextBlack}{HTML}{20242A}

\usepackage[numbers,sort&compress]{natbib}

\usepackage{enumitem}
\setlist{nosep,leftmargin=1.5em}
\usepackage{caption}
\usepackage{fancyhdr}
\usepackage{lastpage}
\usepackage{titlesec}
\usepackage[most]{tcolorbox}
\usepackage{mdframed}
\usepackage[colorlinks=true,linkcolor=AMapDeepBlue,citecolor=AMapDeepBlue,urlcolor=AMapDeepBlue]{hyperref}

\newcommand{\reportname}{LongForcing}
\newcommand{\reportlogo}{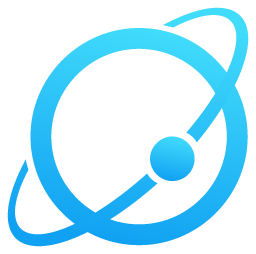}
\newcommand{\reportlogoheight}{18pt}

\newcommand{\orgname}{AMAP CV Lab \textbar{} Alibaba Group}
\newcommand{\reportcenter}[1]{\raisebox{\dimexpr(\depth-\height)/2\relax}{#1}}
\newlength{\reporttextbaselineoffset}
\newcommand{\reporttextalign}[1]{\raisebox{\reporttextbaselineoffset}{#1}}

\renewcommand{\headrulewidth}{1.2pt}
\renewcommand{\footrulewidth}{0.8pt}
\renewcommand{\headrule}{\hbox to\headwidth{\color{AMapBlue}\leaders\hrule height \headrulewidth\hfill}}
\renewcommand{\footrule}{\hbox to\headwidth{\color{AMapLine}\leaders\hrule height \footrulewidth\hfill}}

\titleformat{\section}
  {\Large\bfseries\color{AMapDeepBlue}}
  {\thesection}{0.75em}{}
  [\vspace{0.32em}{\color{AMapLine}\titlerule[0.6pt]}]
\titleformat{\subsection}
  {\large\bfseries\color{AMapDeepBlue}}
  {\thesubsection}{0.65em}{}
\titleformat{\subsubsection}
  {\normalsize\bfseries\color{AMapDeepBlue}}
  {\thesubsubsection}{0.55em}{}
\titlespacing*{\section}{0pt}{1.55em plus 0.25em minus 0.1em}{0.95em}
\titlespacing*{\subsection}{0pt}{1.25em plus 0.2em minus 0.1em}{0.55em}
\titlespacing*{\subsubsection}{0pt}{1.05em plus 0.15em minus 0.08em}{0.45em}

\newtcolorbox{abstractbox}{
  enhanced,
  breakable,
  colback=AMapSky,
  colframe=AMapBlue,
  coltitle=white,
  colbacktitle=AMapBlue,
  title=Abstract,
  fonttitle=\Large\bfseries,
  halign title=center,
  toptitle=5pt,
  bottomtitle=5pt,
  fontupper=\large,
  boxrule=0.8pt,
  arc=1.5mm,
  left=12pt,
  right=12pt,
  top=10pt,
  bottom=10pt,
  before skip=1.2em,
  after skip=0.5em
}

\newmdenv[
  backgroundcolor=gray!8,
  linewidth=0pt,
  innerleftmargin=10pt,
  innerrightmargin=10pt,
  innertopmargin=5pt,
  innerbottommargin=5pt,
  skipabove=5pt,
  skipbelow=5pt
]{cnblock}

\definecolor{DimAction}{HTML}{0877C9}
\definecolor{DimVisual}{HTML}{2E9E5B}
\definecolor{DimPhysics}{HTML}{E07B00}
\definecolor{DimMemory}{HTML}{8E44AD}

\usepackage{algorithm}
\usepackage{algpseudocode}

\usepackage{amsthm}
\newtheorem{definition}{Definition}[section]

\newtheorem{proposition}[definition]{Proposition}
\newtheorem{theorem}[definition]{Theorem}

\usepackage{needspace}
\let\reportsection\section
\renewcommand{\section}{\Needspace{7\baselineskip}\reportsection}
\begin{document}
\thispagestyle{fancy}
\begin{center}
{\bfseries\color{AMapDeepBlue}\LARGE LongForcing: Scaling Temporal Manifold Matching for Long-Horizon World Rollout\par}
\vspace{0.2em}
{\large\color{black}
Ziyu Zhu\textsuperscript{\ensuremath{1,*}}\quad Zhaoxu Sun\textsuperscript{\ensuremath{1,*,\dagger}}\quad Fan Jiang\textsuperscript{\ensuremath{1,*,\ddagger}}\quad Mengchao Wang\textsuperscript{\ensuremath{1,*}}\\
Junfeng Ni\textsuperscript{\ensuremath{1,2,\S}}\quad Ting-Bing Xu\textsuperscript{1}\quad Zhicheng Liu\textsuperscript{1}\quad Mu Xu\textsuperscript{1}\par}
\vspace{-0.1em}
{\small\color{black}\textsuperscript{1}AMAP, Alibaba Group\qquad\textsuperscript{2}Tsinghua University}
\end{center}
\vspace{-1em}
\begingroup
\renewcommand{\thefootnote}{}
\makeatletter
\long\def\@makefntext#1{\noindent#1}
\makeatother
\footnotetext[0]{\textsuperscript{*}Equal contribution.\quad
\textsuperscript{\ensuremath{\dagger}}Project leader.\quad
\textsuperscript{\ensuremath{\ddagger}}Corresponding author.\\
\textsuperscript{\S}Work done during internship at AMAP, Alibaba Group.}
\endgroup

\begin{abstractbox}
Autoregressive video diffusion models have emerged as a foundation for long-horizon world rollout. However, when extended far beyond their training windows, these models suffer from progressive drift: visual quality deteriorates, motion becomes inconsistent, and object identity gradually collapses as generated frames are recursively reused as context.
We study a \emph{supervision-horizon gap}: short-horizon distillation leaves later rollout states without direct supervision. We show that matching short-window marginals need not determine long-horizon behavior, while longer on-policy trajectories expose delayed failures to training. This motivates \emph{temporal manifold matching at scale} for long-horizon world models.
Building on this insight, we introduce \textbf{LongForcing}, a distillation approach that scales teacher supervision to long student self-rollouts. LongForcing preserves the short-horizon teacher-forcing and causal ODE-distillation stages, while replacing the real-score teacher in the final distribution-matching stage with an extended-horizon bidirectional teacher. The long-horizon teacher supervises extended student self-rollouts using the standard distribution-matching formulation, without modifying the student architecture or increasing inference cost.
We evaluate LongForcing on text-to-video generation and interactive action-conditioned world models, achieving a VBench Quality score of 85.52 and improving WorldRoamBench 60s action score by 23.9 points in first-person ablations. Qualitative stress tests further sample interactive rollouts up to \textbf{24 hours}, far beyond the training horizon. These results suggest that scaling supervision along the temporal axis is a practical step toward long-horizon world models.
The project website and code are at \href{https://github.com/Fantasy-AMAP/LongForcing}{https://github.com/Fantasy-AMAP/LongForcing}.

\end{abstractbox}

\suppressfloats[t]
\begin{figure*}[t]
    \centering
    \includegraphics[
        width=\textwidth
    ]{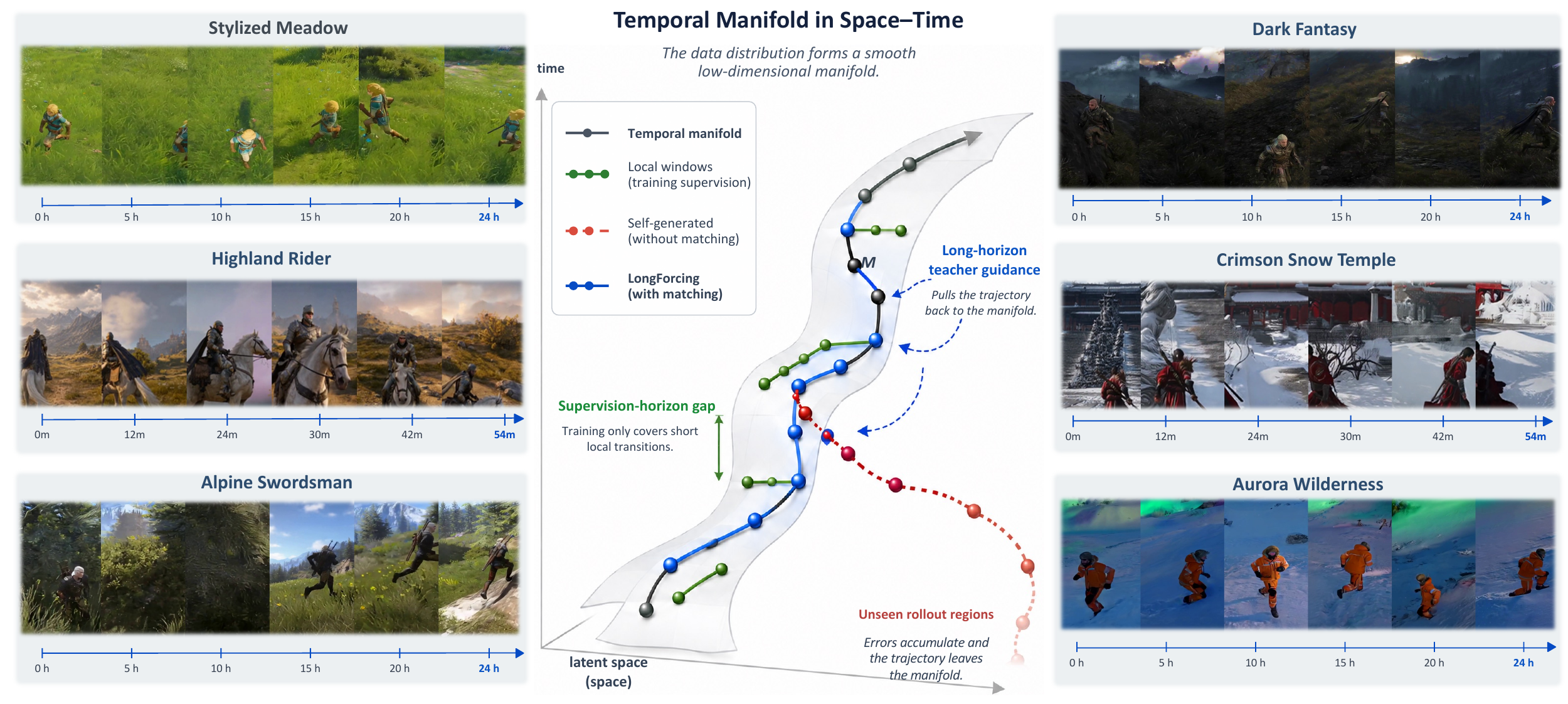}
    \caption{LongForcing extends teacher guidance over student rollouts. Qualitative snapshots from the 5B interactive model span stress-test rollouts sampled up to \textbf{24 hours}.}
    \label{fig:teaser}

\end{figure*}

\setlength{\parskip}{0.5em plus 0.08em minus 0.06em}
\titlespacing*{\paragraph}{0pt}{0.65em}{1em}
\section{Introduction}
\label{sec:introduction}

\begin{center}
    \emph{``We can only see a short distance ahead, but we can see plenty there that needs to be done.''}\\[-1pt]
    --- \textsc{Alan Turing}
\end{center}

Autoregressive (AR) video diffusion models are becoming a foundational paradigm for long-horizon world rollout.
An AR model generates video chunk by chunk, with each new chunk conditioned on previously generated frames.
Repeating this process enables rollouts far longer than individual training clips.
This capability is essential for both interactive world models~\citep{bruce2024genie,google2025genie3,valevski2025gamengen,decart2024oasis,zhang2025matrixgame,sun2025worldplay}, where a world must remain coherent and responsive as the user acts, and long-form text-to-video systems~\citep{kim2024fifo,huang2025selfforcing}, where visual quality, motion, and identity must persist far beyond the initial clip.
However, the same recursive mechanism that enables long rollouts also makes AR models vulnerable to \emph{rollout drift}.
Because every generated chunk becomes context for the next, small errors compound, the conditioning distribution moves away from the data seen during training, and the model gradually loses visual fidelity.

Modern real-time AR video generators are often obtained through a three-stage distillation pipeline.
Teacher forcing first adapts a bidirectional video model to causal conditioning, ODE distillation compresses its denoising trajectory into a few-step generator, and distribution matching distillation (DMD) aligns the student's self-generated distribution with a teacher~\citep{yin2025causvid,huang2025selfforcing,zhu2026causalforcing}.
This recipe is highly effective within the few-second horizon on which it is trained, but deployment requires closed-loop generation over minutes.
Teacher forcing learns from short clean histories, causal ODE distillation uses short teacher-generated trajectories, and short-horizon DMD evaluates only short student rollouts; later rollout states receive no direct supervision from this objective.

Recent approaches mitigate long-horizon drift through sparse local matching~\citep{cui2025selfforcingpp}, inference-time memory and positional modifications~\citep{li2026rolling,yi2025deepforcing,cui2026lol}, or context perturbation~\citep{zhu2025astra}.
Local matching can supervise late rollout states, but its jointly scored windows remain short; inference-time modifications instead stabilize the conditioning context.
Neither mechanism by itself matches the longer joint path, leaving cross-window dependencies that can affect visual fidelity, identity, dynamics, and action consistency.

These failures reveal a \emph{supervision-horizon gap} between short-window training and long-horizon deployment.
Our principle is that \textbf{temporal manifold matching should scale with the trajectories visited on policy}, where temporal manifold informally describes high-probability regions of the teacher's path distribution.
Section~\ref{sec:theory} formalizes this view using path distributions, KL divergence, and total variation.
Matching short windows does not determine the distribution of a long rollout, whereas matching clean trajectories controls drift within the supervised horizon.
Beyond this horizon, stability also depends on teacher leakage, student--teacher mismatch along visited states, and recovery from drifted states.
This motivates extending distribution matching over longer student rollouts while keeping teacher forcing and ODE distillation short.

Guided by this principle, we introduce \textbf{LongForcing}, which replaces Causal Forcing's DMD real-score teacher with an extended-horizon teacher for longer self-generated trajectories.
The student rolls out on policy, while an extended-horizon teacher supplies distribution-level supervision across a long, contiguous portion of each self-generated training trajectory, requiring no architectural changes.
Empirically, LongForcing improves text-to-video VBench Quality by 0.93 points over Causal Forcing and raises WorldRoamBench 60s action score by 23.9 points (37\% relative) in first-person ablations. Qualitative stress tests sample interactive rollouts up to \textbf{24 hours}.
These results position temporal manifold matching at scale as a practical step toward long-horizon world models.

Our contributions are threefold:
\begin{itemize}
    \item We formalize the \textit{supervision-horizon gap}: longer student rollouts expose delayed failures, and joint path matching constrains dependencies left undetermined by local-window marginals.
    \item We introduce \textbf{LongForcing}, which extends the teacher’s supervision horizon in distribution-matching distillation to guide longer student-generated trajectories.
    \item We evaluate LongForcing across interactive world models and text-to-video generation, measuring improvements in action following and video quality alongside their dependence on supervision length.
\end{itemize}

\section{Related Work}

\subsection{Autoregressive Long-Video Generation}

Autoregressive video generation has evolved from discrete-token transformers~\citep{wu2021godiva,hong2022cogvideo,kondratyuk2024videopoet,yan2021videogpt} to diffusion and flow models that stream through temporally varying noise levels, causal attention, or sliding-window denoising~\citep{chen2024diffusionforcing,xie2025progressive,sun2025ardiffusion,teng2025magi,kodaira2026streamdit}.
Few-step variants further use video DMD, student-generated contexts, progressive-window denoising, corrected causal ODE initialization, or diagonal distillation~\citep{yin2025causvid,huang2025selfforcing,liu2025rolling,zhu2026causalforcing,liu2026diagdistill}.
To extend generation beyond the training horizon, training-free methods modify noise schedules, denoising windows, attention, positional encodings, or temporal memory~\citep{qiu2024freenoise,kim2024fifo,yi2025deepforcing,cui2026lol,zhao2026relax,kim2026memrope}, whereas training-based approaches use self-generated segments, long-context tuning, hierarchical denoising, and drift simulation~\citep{cui2025selfforcingpp,yang2025longlive,chen2026contextforcing,zou2026hiar,yuan2026helios}.
Context Forcing~\citep{chen2026contextforcing} also uses a long-context teacher, combining continuation-conditioned DMD with Slow-Fast Memory. LongForcing retains Causal Forcing's student architecture and inference procedure, replacing its DMD teacher to score a longer contiguous, jointly noised student path.

\subsection{Interactive World Models}

Interactive world models use video generation as a controllable simulator for prediction, exploration, and closed-loop interaction.
Recent systems demonstrate playable game simulation~\citep{valevski2025gamengen,decart2024oasis,alonso2024diamond}, camera and action-controlled world exploration~\citep{zhang2025matrixgame,li2025gamecraft,sun2025worldplay}, and real-time interactive generation~\citep{google2025genie3,tencent2025hyworld15,shin2025motionstream}.
Unlike prompt-driven video generation, interactive rollout must preserve both world state and the consequences of continuously changing user actions.
Over extended sessions, accumulated errors can therefore degrade not only visual fidelity and identity, but also action responsiveness and physical consistency.
WorldRoamBench~\citep{xu2026worldroambench} explicitly evaluates these long-horizon properties across diverse interactive environments.
LongForcing scales teacher supervision of on-policy rollouts along the temporal axis in both interactive and text-conditioned generation.
Qualitative stress tests sample interactive rollouts up to \textbf{24 hours}, complementing the benchmark evaluation.

\begin{figure}[!t]
    \centering
    \includegraphics[width=\linewidth]{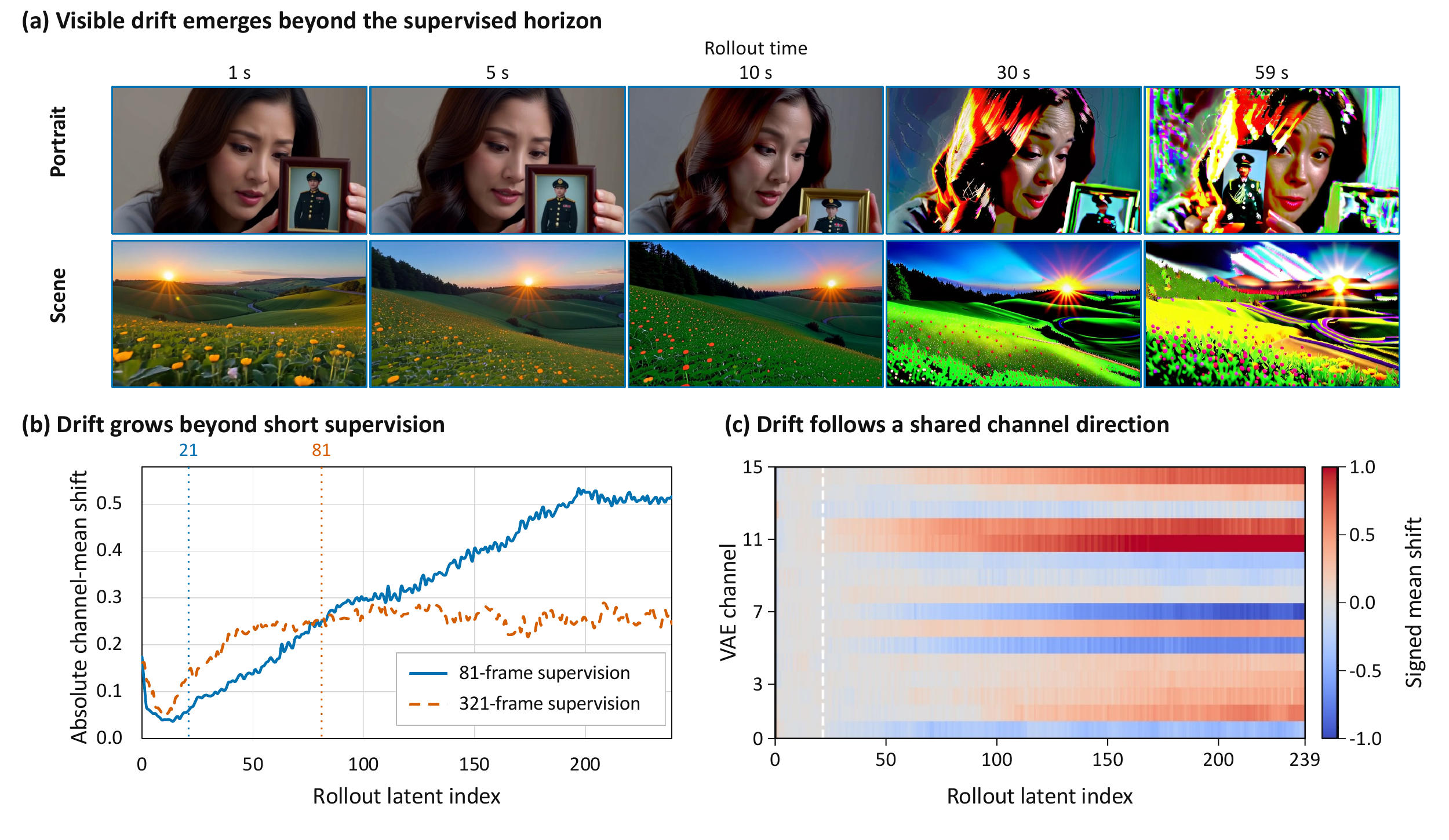}
    \caption{Visual and feature-space signatures of rollout drift. \textbf{(a)} Shared timestamps reveal late over-saturation and high-frequency artifacts. \textbf{(b)} Median absolute channel-mean shift from the early reference; dotted lines mark supervision horizons. \textbf{(c)} Median signed per-channel shift for the 81-frame model; persistent bands indicate a channel-wise direction shared across 20 videos.}
    \label{fig:latent-drift}
\end{figure}

\section{Drift in Autoregressive Video Rollouts}
\label{sec:background}

We first review autoregressive video diffusion and its standard distillation pipeline, then connect the empirical manifestations of long-video drift to the supervision coverage of closed-loop rollout.

\subsection{Autoregressive Video Diffusion and Distillation}
\label{sec:bg-ar-dmd}

\paragraph{Video diffusion.}
Following standard practice, video generation takes place in the latent space of a pretrained spatiotemporal VAE. Let $\mathbf z$ denote a clean video latent. At noise level $s\in[0,1]$, the forward process constructs~\citep{lipman2022flow,zhu2026causalforcing}
$\mathbf z^{(s)}=\alpha_s\mathbf z+\sigma_s\boldsymbol\epsilon$, where $\boldsymbol\epsilon\sim\mathcal N(\mathbf0,\mathbf I)$ and the superscript $(s)$ denotes the diffusion noise level. A score or velocity network learns the reverse dynamics, whose probability-flow ODE maps noise to a clean latent.

\paragraph{Autoregressive rollout.}
Partition a latent video into $K$ causal chunks $\mathbf z_{1:K}$, where each chunk is generated from the preceding chunks and condition $\mathbf c$. The resulting path distribution factorizes as $p_\theta(\mathbf z_{1:K}\mid\mathbf c)=\prod_{k=1}^{K}p_\theta(\mathbf z_k\mid\mathbf z_{<k},\mathbf c)$, where $\mathbf c$ denotes a text prompt or a prompt augmented with per-chunk actions.
At inference, each prediction is appended to the history, enabling streaming while accumulating error.

\paragraph{Few-step distillation and distribution matching.}
An efficient causal student is typically obtained through teacher forcing on short clean histories, causal ODE distillation into a few-step generator, and DMD refinement~\citep{zhu2026causalforcing,yin2024dmd,yin2024improved}.
DMD compares bidirectional real and fake scores on student-generated closed-loop paths; Self Forcing~\citep{huang2025selfforcing} applies this refinement on policy, but only within a short rollout window.

\subsection{Why Autoregressive Rollouts Drift}
\label{sec:bg-drift}

Long-video drift can appear as position or color shift, amplified blur and noise, and action inconsistency~\citep{yuan2026helios}. These mechanisms share a temporal structure: generated chunks recursively change the conditioning distribution. The relevant train--test mismatch is therefore between the short prefixes evaluated during training and the evolving on-policy path traversed during long deployment.

Fix a prefix-measurable drift detector $g_t(\mathbf z_{1:t})\in[0,1]$ and threshold $\gamma$, and define the first drift time and within-horizon event
\begin{equation}
\label{eq:drift-time}
\tau_{\rm drift}=\inf\{t\ge1:g_t>\gamma\},\qquad \mathcal B_H(\gamma)=\{\tau_{\rm drift}\le H\}.
\end{equation}
For a fixed rollout distribution $Q$ and $H_1<H_2$, lengthening the path exposes exactly the delayed failure mass
\begin{equation}
\label{eq:exposed-mass}
\Pr_Q(\tau_{\rm drift}\le H_2)-\Pr_Q(\tau_{\rm drift}\le H_1)=\Pr_Q(H_1<\tau_{\rm drift}\le H_2).
\end{equation}
Thus an objective computed only from $H_1$-step prefixes cannot evaluate failures whose first hitting time exceeds $H_1$. A longer rollout changes which delayed states can receive a teacher signal; whether optimization corrects them is addressed in Section~\ref{sec:theory}.

Figure~\ref{fig:latent-drift} reveals structured late-rollout shift in the shared causal-VAE latent space: the channel-mean deviation grows beyond the short supervision horizon, while the signed heat map shows a consistent channel-wise direction. Appendix~\ref{app:drift-metric} gives the exact reference window and aggregation procedure.

\setlength{\parskip}{0.70em plus 0.08em minus 0.06em}
\titlespacing*{\paragraph}{0pt}{3.25ex plus 1ex minus .2ex}{1em}

\section{Finite-Horizon Drift Control}
\label{sec:theory}

We separate three questions behind LongForcing: which failures can be detected within a finite training window, what matching trajectories within that window guarantees, and under what conditions those guarantees extend to later blocks.
Throughout, $H$ denotes a horizon of $H$ autoregressive chunks; experiments report the corresponding number of video frames. Appendix~\ref{app:theory} gives the complete setup, proofs, and scope conditions.

\subsection{The Supervision-Horizon Gap}
\label{sec:gap}

Let $S_t$ denote the student's causal state after chunk $t$, and let $\mathcal C_H=(S_0,\mathbf a_{1:H},\mathbf c)$ collect the initial state, an action sequence fixed over the rollout, and persistent conditions such as the text prompt. For this shared condition, $Q_{\theta,H}$ denotes the causal student's $H$-chunk path distribution and $Q^{*,\mathrm{long}}_{\phi,H}$ the extended-horizon teacher's joint reference. The distinction between matching local windows and matching this joint path is fundamental.

\begin{proposition}[Short windows do not identify the long path]
\label{prop:indistinguishable}
For every $W<T$, two causal path distributions can agree on every contiguous $W$-step marginal while having disjoint $T$-step supports. Hence even perfect local-window matching need not determine the long trajectory distribution.
\end{proposition}
\emph{Proof in Appendix~\ref{app:proof-indistinguishable}.}
Proposition~\ref{prop:indistinguishable} is an identifiability statement: local-window marginals leave cross-window coupling undetermined. Equation~\eqref{eq:exposed-mass} gives a separate exposure statement. For a fixed rollout law, an $H$-chunk prefix objective cannot evaluate failures first reached after $H$; extending that rollout exposes exactly the delayed-failure mass in the added suffix. This exposure makes those states available for supervision, but does not imply monotonic improvement across models.

\subsection{What Path Matching Guarantees}
\label{sec:closure}

For a shared condition distribution $\mu_H$, let $\overline Q_{\theta,H}=\mu_HQ_{\theta,H}$ and $\overline Q^*_{\phi,H}=\mu_HQ^{*,\mathrm{long}}_{\phi,H}$. The ideal clean path-matching objective is
\begin{equation}
\label{eq:ideal}
\mathcal L_H^{\mathrm{ideal}}
=\mathbb E_{\mathcal C_H\sim\mu_H}
D_{\mathrm{KL}}\!\left(Q_{\theta,H}(\cdot\mid\mathcal C_H)\,\middle\|\,Q^{*,\mathrm{long}}_{\phi,H}(\cdot\mid\mathcal C_H)\right).
\end{equation}
\begin{theorem}[Within-horizon drift control]
\label{thm:within}
For every measurable path event $\mathcal A_H$,
\begin{equation}
\label{eq:within}
\Pr_{\overline Q_{\theta,H}}(\mathcal A_H)
\le \Pr_{\overline Q^*_{\phi,H}}(\mathcal A_H)
+\sqrt{\mathcal L_H^{\mathrm{ideal}}/2}.
\end{equation}
\end{theorem}
\emph{Proof in Appendix~\ref{app:proof-within}.}
Taking $\mathcal A_H=\mathcal B_H(\gamma)$ bounds student drift by teacher drift plus the path-matching slack. Increasing $H$ therefore enlarges the delayed events covered by the guarantee, but does not automatically shrink the slack or control failures after $H$.

\subsection{Teacher-Guided Correction over Finite Depth}
\label{sec:horizon-comparison}

For a fixed evaluation stride $B$, let $\mathcal D$ be a measurable degraded region and $\mathcal G=\mathcal D^c$; when the drift detector is state-measurable, $\mathcal D$ can be its thresholded region from Eq.~\eqref{eq:drift-time}. Define
$b_j=\Pr(S_{jB}\in\mathcal D)$,
$a_j=\Pr(S_{(j+1)B}\in\mathcal D\mid S_{jB}\in\mathcal G)$, and
$\eta_j=\Pr(S_{(j+1)B}\in\mathcal G\mid S_{jB}\in\mathcal D)$.
Total probability gives
\begin{equation}
\label{eq:recursion-exact}
 b_{j+1}=(1-b_j)a_j+(1-\eta_j)b_j.
\end{equation}
Let $\lambda_j^*$ be teacher leakage on the student's on-policy good-start distribution, $\eta_j^*$ the corresponding teacher return from degraded starts, and $\epsilon_j^{\mathcal G},\epsilon_j^{\mathcal D}$ the clean $B$-step student--teacher path KLs on those starts. Conditional Pinsker yields
\begin{equation}
\label{eq:block-control}
a_j\le\lambda_j^*+\sqrt{\epsilon_j^{\mathcal G}/2},
\qquad
\eta_j\ge\left[\eta_j^*-\sqrt{\epsilon_j^{\mathcal D}/2}\right]_+.
\end{equation}

\begin{proposition}[Conditional finite-depth drift control]
\label{thm:recursion}
Suppose for $0\le j<J$ that $\lambda_j^*\le\bar\lambda$, $\eta_j^*\ge\bar\eta^*$, $\epsilon_j^{\mathcal G}\le\bar\epsilon_{\mathcal G}$, and $\epsilon_j^{\mathcal D}\le\bar\epsilon_{\mathcal D}$. Define
$A=\min\{1,\bar\lambda+\sqrt{\bar\epsilon_{\mathcal G}/2}\}$ and
$\underline\eta=[\bar\eta^*-\sqrt{\bar\epsilon_{\mathcal D}/2}]_+>0$. Then
\begin{equation}
\label{eq:recursion-bound}
b_j\le\min\!\left\{1,(1-\underline\eta)^j b_0+
\frac{A}{\underline\eta}\left[1-(1-\underline\eta)^j\right]\right\},
\quad 0\le j\le J.
\end{equation}
\end{proposition}
\emph{Proof in Appendix~\ref{app:proof-recursion}.}
Here $A$ controls newly injected degradation, while $\underline\eta$ lower-bounds student recovery using the teacher's return probability and the degraded-start path mismatch. Equation~\eqref{eq:recursion-bound} controls degraded-state occupancy at these block boundaries and is nontrivial when $A<\underline\eta$. These are sufficient conditions, not guarantees supplied automatically by LongForcing; feedback actions and a stronger conditional long-run extension are treated in Appendix~\ref{app:finite-depth-details}--\ref{app:long-run}.

\section{LongForcing}
\label{sec:method}

LongForcing builds on Causal Forcing~\citep{zhu2026causalforcing}. We retain its short-horizon teacher-forcing and causal ODE-distillation stages, and replace the real-score teacher in the final asymmetric DMD stage with an extended-horizon bidirectional teacher.
This teacher supervises longer student self-rollouts using the standard DMD formulation, motivated by the analysis in Section~\ref{sec:theory}.

\subsection{Long-Horizon Trajectory Matching}
\label{sec:longforcing}

\paragraph{Objective.}
With the long-horizon teacher as the reference, we apply the standard noisy DMD objective to an $H$-chunk student path, extending the short supervision horizon $W$:
\begin{equation}
\label{eq:longdmd}
    \mathcal{J}_H^{\mathrm{noisy}}
    =
    \mathbb{E}_{\mathcal C_H\sim\mu_H,\,s}
    \left[
        D_{\mathrm{KL}}\Big(
            Q^{(s)}_{\theta,H}(\cdot \mid \mathcal{C}_H)
            \,\Big\|\,
            Q^{*,(s)}_{\phi,H}(\cdot \mid \mathcal{C}_H)
        \Big)
    \right],
\end{equation}
where $Q^{(s)}_{\theta,H}$ and $Q^{*,(s)}_{\phi,H}$ are the student and teacher path distributions at shared noise level $s$.
As in asymmetric DMD, the student generates causally and both score networks evaluate its noised path bidirectionally. The longer scoring horizon covers later prefixes; Appendix~\ref{app:longforcing-objective}--\ref{app:dmd-estimator} gives the objective's scope and score estimator.

\paragraph{Extended-horizon teacher.}
The frozen real-score model is trained or fine-tuned on long sequences so that it can score the complete $H$-step path; the online fake-score model uses the same long bidirectional context to track the current student distribution.
For interactive training, a pre-specified action plan is supplied over the same horizon.
The extended horizon increases supervision coverage; the guarantees remain conditional on teacher quality and student--teacher mismatch.

\paragraph{Algorithm.}
Algorithm~\ref{alg:longforcing} summarizes the pipeline. In Stage~3, the long-horizon real-score teacher supervises the longer self-generated path.

\begin{algorithm}[H]
\caption{LongForcing training}
\label{alg:longforcing}
\begin{algorithmic}[1]
\Require short horizon $W$; long horizon $H$; extended teacher $s_{\mathrm{real}}$
\State Train a causal model by teacher forcing on $W$-step real clips.
\State Use this causal model as the ODE teacher to initialize $G_\theta$ on $W$-step clips.
\State Initialize the online model $s_{\mathrm{fake}}$ from $s_{\mathrm{real}}$.
\For{each LongForcing iteration}
    \State Sample $\mathcal C_H$ and roll out $\mathbf z_{1:H}\sim Q_{\theta,H}(\cdot\mid\mathcal C_H)$ on policy.
    \State Sample $s$ and noise the full path to obtain $\mathbf z_{1:H}^{(s)}$.
    \State Evaluate $s_{\mathrm{real}}$ and $s_{\mathrm{fake}}$ bidirectionally on $\mathbf z_{1:H}^{(s)}$.
    \State Update $G_\theta$ using the DMD estimator and $s_{\mathrm{fake}}$ on their respective schedules (Appendix~\ref{app:training-pipeline}).
\EndFor
\State \Return few-step causal student $G_\theta$
\end{algorithmic}
\end{algorithm}

\section{Experiments}
\label{sec:experiments}
\subsection{Setup}
\label{sec:exp-setup}

\paragraph{Interactive world model.}
We train action-conditioned models on Wan2.2-5B~\citep{wan2025wan22repo} and Wan2.1-1.3B~\citep{wan2025wan} and evaluate 60-second WorldRoamBench rollouts: third-person for the 5B comparison and first-person for the 1.3B horizon ablation. Teacher forcing and ODE distillation use 121 and 81 frames for the 5B and 1.3B models, respectively, while DMD extends both to 321 frames. DMD training uses AdamW on eight GPUs with a global batch size of 8 for 1000 steps. The generator and fake-score network use constant learning rates of $2\times10^{-6}$ and $4\times10^{-7}$, respectively, with $(\beta_1,\beta_2)=(0,0.999)$ and weight decay $0.01$.

\paragraph{Text-to-video model.}
For text-to-video, we follow Causal Forcing~\citep{zhu2026causalforcing} on Wan2.1-1.3B~\citep{wan2025wan} at $832{\times}480$ resolution. We retain its teacher-forcing and causal ODE configurations, replace the DMD real-score teacher, and use 321-frame DMD rollouts. DMD uses AdamW on 8 GPUs with a global batch size of 8 and learning rates of $2{\times}10^{-6}$ and $4{\times}10^{-7}$ for the generator and fake-score network. We evaluate with the official 5-second VBench protocol~\citep{huang2024vbench} and VBench-Long~\citep{huang2025vbenchpp} on 30-second generations. See Appendix~\ref{app:training-pipeline} for further details.

\subsection{Main Results}
\label{sec:main-results}

\paragraph{Interactive world model.}
Table~\ref{tab:main-world} reports control-rate-weighted scores for 60-second third-person rollouts. Scores incorporate control rate as defined in the table caption. LongForcing-5B achieves weighted action and visual scores of 76.81 and 57.89, the highest among the open-source models shown. Its control rate is 87.88\% (87/99), compared with 40.40\% for LingBot-World v2.

The 47.48-percentage-point control-rate advantage over LingBot-World v2 highlights the importance of sustained action responsiveness in the weighted evaluation.
The scores therefore reflect controllability and visual quality jointly, rather than visual quality alone.
Achieving this result with a 5B model against 8B and 14B baselines also suggests that parameter count alone does not determine long-rollout performance; Section~\ref{sec:ablation} examines supervision configurations within a fixed backbone.

\begin{table}[H]
\centering
\caption{\textbf{WorldRoamBench results.}
Except Ctrl., higher-is-better scores are multiplied by $p=\mathrm{Ctrl.}/100$; Aesth. Drift is $100[1-p(1-d)]$ for normalized drift $d\in[0,1]$. Best and second-best open-source results are bold and underlined.}
\label{tab:main-world}
\scriptsize
\setlength{\tabcolsep}{2.6pt}
\resizebox{\textwidth}{!}{%
\begin{tabular}{@{}lccccccccc@{}}
\toprule
\multirow{2}{*}{Model} & \multirow{2}{*}{Size} & \multicolumn{4}{c}{WASD Action Interaction} & \multicolumn{4}{c}{Visual Quality} \\
\cmidrule(lr){3-6}\cmidrule(lr){7-10}
& & Action $\uparrow$ & Ctrl. $\uparrow$ & Partial Acc. $\uparrow$ & Traj. $\uparrow$ & Visual $\uparrow$ & Aesthetic $\uparrow$ & Aesth. Drift $\downarrow$ & Imaging $\uparrow$ \\
\midrule
\multicolumn{10}{@{}l}{\textit{Proprietary Models}} \\
Genie 3~\citep{google2025genie3} & -- & 69.73 & 79.80 & 65.75 & 73.70 & 56.26 & 48.11 & 32.08 & 45.30 \\
HappyOyster~\citep{alibaba2026happyoyster} & -- & 84.02 & 96.97 & 80.48 & 87.57 & 65.63 & 59.10 & 20.61 & 47.44 \\
\midrule
\multicolumn{10}{@{}l}{\textit{Open-Source Models}} \\
LingBot-World~\citep{robbyant2026lingbotworld} & 14B & 18.81 & 23.23 & 18.50 & 19.12 & 15.99 & 14.32 & 81.70 & 13.84 \\
HY-World 1.5~\citep{tencent2025hyworld15} & 8B & 24.92 & 29.29 & 24.57 & 25.28 & 19.11 & 17.52 & 79.48 & 16.52 \\
LingBot-World v2~\citep{gao2026lingbotworldv2} & 14B & \underline{36.09} & \underline{40.40} & \underline{35.35} & \underline{36.84} & \underline{28.42} & \underline{25.61} & \underline{68.75} & \underline{25.04} \\
\midrule
\textbf{LongForcing (ours)} & 5B & \textbf{76.81} & \textbf{87.88} & \textbf{73.74} & \textbf{79.89} & \textbf{57.89} & \textbf{51.69} & \textbf{31.93} & \textbf{48.26} \\
\bottomrule
\end{tabular}
}
\end{table}

\paragraph{Text-to-video: 5-second VBench.}
Table~\ref{tab:main-t2v} reports 5-second VBench results.
At matched inference cost, LongForcing leads distilled autoregressive models with 84.56 Total and 85.52 Quality, improving over Causal Forcing by 0.52 and 0.93 points, respectively.
These results suggest that longer-horizon supervision also benefits short-clip visual quality, extending its utility beyond long-rollout stability.

\paragraph{Text-to-video: 30-second VBench-Long.}
Table~\ref{tab:main-t2v-long} reports 30-second VBench-Long results, where LongForcing achieves the highest Total and Quality scores among the compared methods, at 82.83 and 83.99, improving over Causal Forcing by 1.78 and 0.55 points, respectively. The full results in Appendix~\ref{app:full-vbench-results} show the highest background consistency score of 96.58. Semantic alignment also improves from 71.52 to 78.19, complementing the gains in visual quality and background consistency.

\begin{table}[H]
\centering
\begin{minipage}[t]{0.48\textwidth}
\centering
\captionof{table}{\textbf{5-second VBench.}}
\label{tab:main-t2v}
\small
\setlength{\tabcolsep}{4.5pt}
\resizebox{\linewidth}{!}{%
\begin{tabular}{@{}lccc@{}}
\toprule
Model & Total $\uparrow$ & Quality $\uparrow$ & Semantic $\uparrow$ \\
\midrule
CausVid~\citep{yin2025causvid} & 81.33 & 83.98 & 70.72 \\
Self Forcing~\citep{huang2025selfforcing} & 83.74 & 84.48 & \underline{80.77} \\
Causal Forcing~\citep{zhu2026causalforcing} & \underline{84.04} & \underline{84.59} & \textbf{81.84} \\
\textbf{LongForcing (ours)} & \textbf{84.56} & \textbf{85.52} & 80.72 \\
\bottomrule
\end{tabular}
}
\end{minipage}
\hfill
\begin{minipage}[t]{0.48\textwidth}
\centering
\captionof{table}{\textbf{30-second VBench-Long.}}
\label{tab:main-t2v-long}
\small
\setlength{\tabcolsep}{4.5pt}
\resizebox{\linewidth}{!}{%
\begin{tabular}{@{}lccc@{}}
\toprule
Model & Total $\uparrow$ & Quality $\uparrow$ & Semantic $\uparrow$ \\
\midrule
LongLive~\citep{yang2025longlive} & \underline{82.80} & 83.21 & \textbf{81.17} \\
Self Forcing~\citep{huang2025selfforcing} & 81.49 & 82.67 & 76.77 \\
Causal Forcing~\citep{zhu2026causalforcing} & 81.05 & \underline{83.44} & 71.52 \\
\textbf{LongForcing (ours)} & \textbf{82.83} & \textbf{83.99} & \underline{78.19} \\
\bottomrule
\end{tabular}
}
\end{minipage}
\end{table}

\subsection{Ablation Studies}
\label{sec:ablation}

\paragraph{Text-to-video supervision horizon.}
Table~\ref{tab:abl-horizon} compares supervision-horizon configurations from 81 to 321 frames (approximately 5 to 20 seconds). Each DMD rollout length uses its corresponding separately trained teacher checkpoint.
The 321-frame setting achieves the highest 5-second total score and 30-second total, quality, and semantic scores.

\begingroup
\setlength{\intextsep}{6pt plus 2pt minus 2pt}
\begin{table}[H]
\centering
\caption{\textbf{Supervision-horizon configurations for text-to-video generation.}
Horizons are rollout lengths in frames. Best and second-best scores are \textbf{bold} and \underline{underlined}.
}
\label{tab:abl-horizon}
\setlength{\tabcolsep}{3.8pt}
\small
\begin{tabular}{@{}cc c cccc@{}}
\toprule
\multirow{2}{*}{Frames} & \multirow{2}{*}{Approx. span} & \multirow{2}{*}{\shortstack{VBench (5s)\\Total $\uparrow$}} & \multicolumn{4}{c}{VBench-Long (30s) $\uparrow$} \\
\cmidrule(l){4-7}
& & & Total & Quality & Semantic & Dynamic \\
\midrule
81  & $\sim$5s  & 84.36 & \underline{81.87} & 83.71 & \underline{74.54} & 58.15 \\
161 & $\sim$10s & 83.76 & 78.02 & 81.59 & 63.77 & 56.67 \\
241 & $\sim$15s & \underline{84.38} & 79.18 & \underline{83.75} & 60.87 & \textbf{94.91} \\
321 & $\sim$20s & \textbf{84.56} & \textbf{82.83} & \textbf{83.99} & \textbf{78.19} & \underline{60.93} \\
\bottomrule
\end{tabular}
\end{table}

\paragraph{Interactive world-model supervision horizon.}
In Table~\ref{tab:abl-world-horizon}, extending the DMD supervision horizon from 81 to 321 frames raises action score from 64.39 to 88.26 and strict accuracy from 42.89 to 85.16 on 60-second rollouts, while reducing aesthetic drift from 56.58 to 42.37. The 321-frame setting achieves the best results across all action metrics and most visual metrics, supporting the benefit of longer supervision in exposing accumulated control and consistency errors.

\begin{table}[H]
\centering
\caption{\textbf{Effect of the DMD supervision horizon on first-person interactive world modeling.}
The comparison varies the DMD supervision horizon with a fixed 1.3B backbone under the first-person WorldRoamBench 60s protocol; scores are unweighted.
Higher is better except for Aesthetic Drift.
The best and second-best results are shown in \textbf{bold} and \underline{underlined}, respectively.
}
\label{tab:abl-world-horizon}
\scriptsize
\setlength{\tabcolsep}{2.4pt}
\resizebox{\textwidth}{!}{%
\begin{tabular}{@{}cccccccccc@{}}
\toprule
\multirow{2}{*}{Frames} & \multirow{2}{*}{Approx. span} & \multicolumn{4}{c}{WASD Action Interaction} & \multicolumn{4}{c}{Visual Quality} \\
\cmidrule(lr){3-6}\cmidrule(lr){7-10}
& & Action $\uparrow$ & Strict Acc. $\uparrow$ & Partial Acc. $\uparrow$ & Traj. $\uparrow$ & Visual $\uparrow$ & Aesthetic $\uparrow$ & Aesth. Drift $\downarrow$ & Imaging $\uparrow$ \\
\midrule
81  & $\sim$5s  & 64.39 & 42.89 & 63.01 & \underline{87.27} & 45.05 & 39.81 & 56.58 & 46.40 \\
161 & $\sim$10s & 56.16 & 43.21 & 50.00 & 75.28 & 47.01 & 41.06 & 60.18 & \textbf{52.29} \\
241 & $\sim$15s & \underline{83.39} & \underline{78.24} & \underline{85.14} & 86.78 & \underline{55.19} & \underline{49.93} & \underline{43.30} & 50.25 \\
321 & $\sim$20s & \textbf{88.26} & \textbf{85.16} & \textbf{90.56} & \textbf{89.05} & \textbf{56.30} & \textbf{51.61} & \textbf{42.37} & \underline{50.84} \\
\bottomrule
\end{tabular}
}
\end{table}

\paragraph{Time-resolved rollout drift.}
Figure~\ref{fig:drift-human-study} shows how text-to-video generation deteriorates as rollouts grow longer. Extending supervision from 81 to 321 frames reduces the mean drift score at 59 seconds from 67.46 to 31.91, a relative reduction of 52.7\%. The gap between the curves grows over time, indicating that longer supervision slows error accumulation and helps maintain quality during extended generation, although some drift remains. Appendix~\ref{app:drift-score} provides the metric definition, evaluation protocol, and corresponding interactive results.

\begin{figure}[H]
    \centering
    \includegraphics[width=\textwidth]{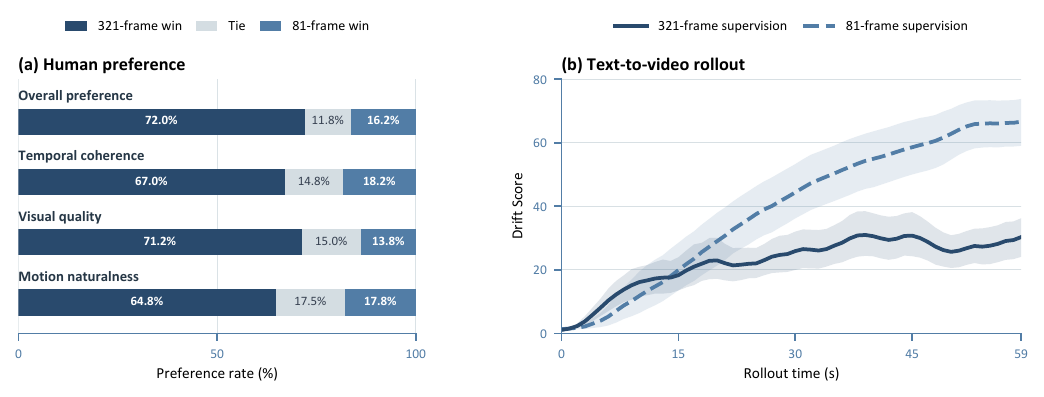}
    \caption{\textbf{Rollout drift and human preference across supervision horizons.}
    \textbf{(a)} Blinded pairwise preferences.
    \textbf{(b)} Mean Drift Score.}
    \label{fig:drift-human-study}
\end{figure}
\endgroup

\paragraph{Human study.}
Across 400 judgments per criterion from 20 evaluators on 20 shared cases, 321-frame LongForcing receives 72.0\% overall preference, compared with 16.2\% for the 81-frame model and 11.8\% ties (Figure~\ref{fig:drift-human-study}).
Preferences also favor longer supervision in visual quality by 71.2\% to 13.8\%, temporal coherence by 67.0\% to 18.2\%, and motion naturalness by 64.8\% to 17.8\%.
Appendix~\ref{app:human-study} gives the evaluation protocol.

\subsection{Qualitative Results}
\label{sec:qualitative}
Figure~\ref{fig:qualitative-interactive}A--D shows recognizable scenes and characters across snapshots spanning up to 54 minutes.
In the 600-second comparisons (E--F), LingBot-World v2 (14B)~\citep{gao2026lingbotworldv2} exhibits over-saturation and repetitive textures (E), and scene and character appearance shifts (F), while LongForcing-5B retains more of the initial scene structure and visual style.
These comparisons assess visual consistency at the displayed timestamps.
\begin{figure}[H]
    \centering
    \includegraphics[width=0.96\textwidth]{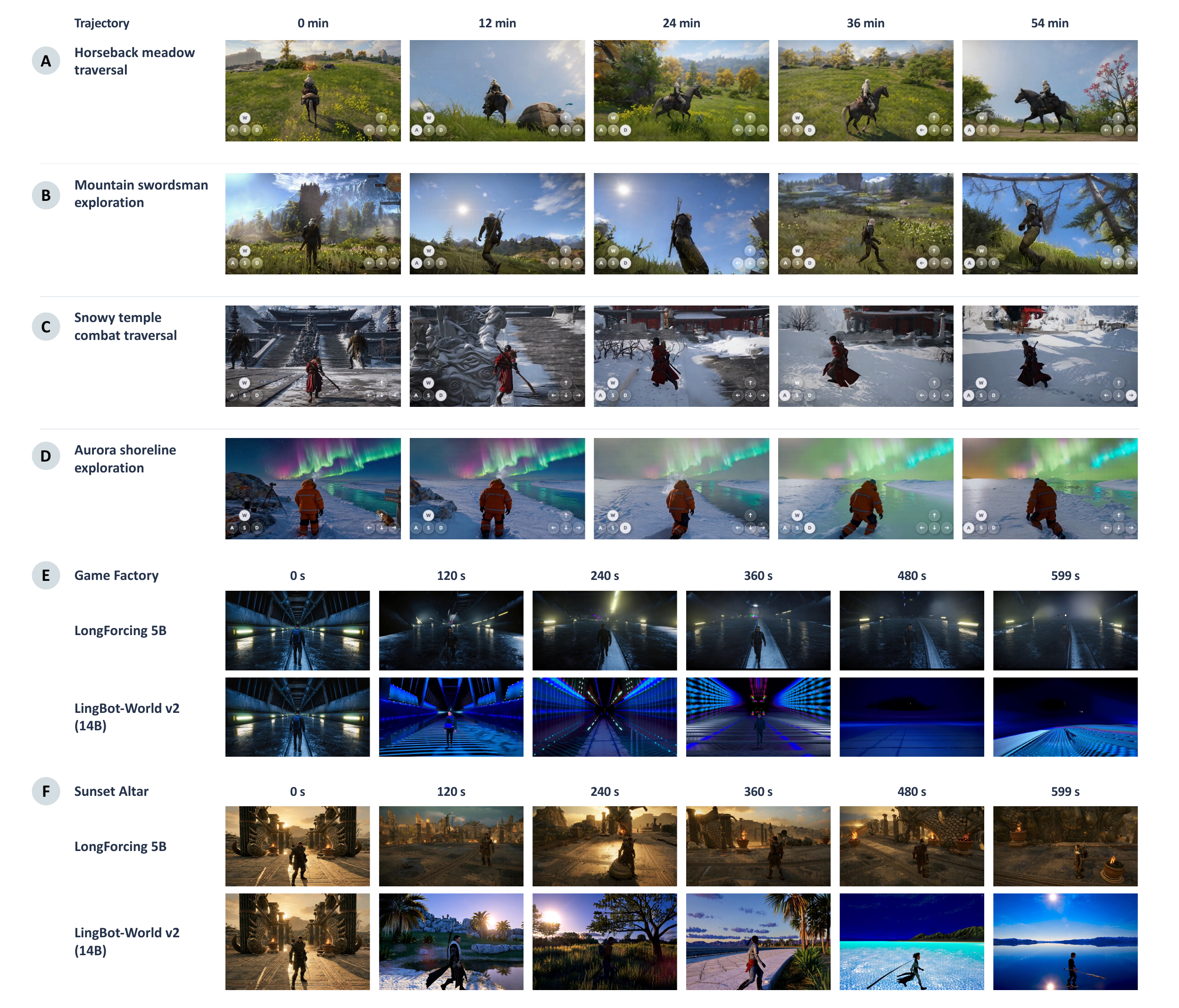}
    \caption{A--D: LongForcing rollouts over up to 54 minutes. E--F: LongForcing-5B (top) and LingBot-World v2 (14B; bottom) at matched timestamps from 0 to 599 seconds.}
    \label{fig:qualitative-interactive}
\end{figure}

\section{Conclusion}

We identify a \emph{supervision-horizon gap} in autoregressive video generation: short-horizon supervision leaves later rollout behavior underdetermined, while longer on-policy trajectories expose delayed failures to teacher guidance.
LongForcing addresses this gap by replacing Causal Forcing's DMD real-score teacher with an extended-horizon teacher that supervises longer student self-rollouts, while retaining its short teacher-forcing and causal ODE-distillation stages.
Across text-to-video generation and interactive world models, this improves VBench Quality by 0.93 points and first-person action score by 23.9 points over the respective baselines, without changing the architecture or increasing inference cost.
Horizon ablations and time-resolved evaluations show reduced rollout drift, complemented by qualitative interactive stress tests sampled up to 24 hours.
These findings position temporal supervision as a scaling dimension for video world models.
Integrating coherent visual rollouts with persistent memory and hierarchical planning offers a path toward persistent interactive worlds, goal-directed behavior, and long-term reasoning.

\bibliographystyle{unsrt}
\bibliography{references}

\input{appendix}

\end{document}

%% file: appendix.tex
\clearpage
\appendix
\setcounter{figure}{0}
\setcounter{equation}{0}
\makeatletter
\@addtoreset{figure}{section}
\@addtoreset{equation}{section}
\makeatother
\renewcommand{\thefigure}{\Alph{section}\arabic{figure}}
\renewcommand{\theHfigure}{\Alph{section}.\arabic{figure}}
\renewcommand{\theequation}{\Alph{section}\arabic{equation}}
\renewcommand{\theHequation}{\Alph{section}.\arabic{equation}}

\section{Implementation Details}

\subsection{Training Pipeline}
\label{app:training-pipeline}

LongForcing retains the teacher-forcing and causal ODE-distillation stages of Causal Forcing and replaces the real-score teacher in its final DMD stage with an extended-horizon bidirectional teacher.
Teacher forcing uses short clean histories and causal ODE distillation uses short trajectories from the causal teacher; the resulting student generates a contiguous $H$-chunk closed-loop path, scored jointly by the replacement teacher and the online fake-score network.
The student architecture, sampler, and inference procedure remain unchanged.
Under the causal VAE mapping, 81, 161, 241, and 321 decoded RGB frames correspond to 21, 41, 61, and 81 temporal latent positions, respectively; these positions are rounded and grouped into autoregressive chunks of three latent frames.

\paragraph{Interactive teacher: initialization and data.}
The interactive teachers use the pretrained Wan2.2-5B and Wan2.1-1.3B backbones, adapted to action-conditioned video generation on the ABot-World training data~\citep{jiang2026abotworld}.
We train the interactive teachers for 20,000 steps with a learning rate of $10^{-5}$ and a global batch size of 256 video clips.
We use action-annotated video sequences, with the action sequence aligned to the video over the teacher's temporal window.
The teacher retains bidirectional temporal attention so that it can evaluate the complete student-generated path with its corresponding actions.
The data collection and annotation pipeline is described in ABot-World~\citep{jiang2026abotworld}; the long teacher provides the frozen real-score network for our 321-frame DMD stage.

\paragraph{Text-to-video teacher: long-video adaptation.}
We initialize the bidirectional teacher from the pretrained Wan2.1-T2V-1.3B checkpoint~\citep{wan2025wan} and fine-tune it on UltraVideo-Long~\citep{xue2025ultravideo} for three epochs, with each training example visited once per epoch.
Videos are sampled at 16 FPS and processed at $832\times480$ resolution, using 321-frame clips for the long teacher.
The horizon study uses separate teacher checkpoints for the 81-, 161-, 241-, and 321-frame DMD configurations, pairing each rollout length with its corresponding teacher.
The pretrained Wan VAE and UMT5 text encoder remain frozen; video latents and text embeddings are cached before optimization.
We optimize a supervised flow-matching loss on the noisy video latents, sampling a diffusion timestep and Gaussian noise anew for each example and applying the scheduler's timestep-dependent loss weight.
Adaptation uses rank-64 LoRA with scaling parameter $\alpha=64$ in the DiT attention projections (\texttt{q,k,v,o}) and feed-forward projections (\texttt{ffn.0,ffn.2}); the underlying pretrained weights remain frozen.
AdamW uses a base learning rate of $10^{-4}$, weight decay $0.01$, and its default momentum parameters $(0.9,0.999)$.
The implementation uses BF16 model parameters, evaluates the mean-squared-error loss in FP32, enables gradient checkpointing, and processes one clip per GPU without gradient accumulation, giving a global batch size of 32 clips.
After adaptation, the resulting teacher is frozen during DMD; only the causal generator and the online fake-score network are updated in that stage.

\paragraph{Resources and DMD configuration.}
We follow Causal Forcing~\citep{zhu2026causalforcing} for DMD timestep sampling and loss weighting.
LongForcing extends the real-score teacher's supervision horizon and the student rollout length.
Interactive teacher training uses 256 GPUs, while text-to-video teacher training uses 32 NVIDIA H20 GPUs.
The subsequent DMD stage runs for 1,000 global steps on eight GPUs with a global batch size of eight trajectories in both text-to-video and interactive experiments.
The generator and fake-score learning rates are $2\times10^{-6}$ and $4\times10^{-7}$, respectively, with AdamW momentum parameters $(0,0.999)$ and weight decay $0.01$.
For text-to-video DMD training, the fake-score network is updated at every global step, whereas the generator is updated once every five global steps, giving a $5{:}1$ ratio of fake-score to generator updates.
For both the 1.3B and 5B interactive world models, the fake-score network is updated at every global step and the generator once every three global steps during DMD training.

\subsection{Ideal Objective and Practical Scope}
\label{app:longforcing-objective}

The ideal clean objective assigns weight only to prefixes reached inside the $H$-step training rollout.
Let $q_{\theta,t}$ and $q^*_{\phi,t}$ denote the student and teacher chain-rule conditionals induced by their respective joint path distributions.
Then
\begin{equation}
\label{eq:path-chain-rule}
    \mathcal L_H^{\mathrm{ideal}}
    =\sum_{t=1}^H
    \mathbb E_{\substack{\mathcal C_H\sim\mu_H,\\
                    \mathbf z_{<t}\sim Q_{\theta,H}}}
    D_{\mathrm{KL}}\!\left(
        q_{\theta,t}(\cdot\mid\mathbf z_{<t},\mathcal C_H)
        \,\middle\|\,
        q^*_{\phi,t}(\cdot\mid\mathbf z_{<t},\mathcal C_H)
    \right).
\end{equation}
The expectation is under the student's on-policy prefix distribution; the teacher conditional is derived from its joint distribution and need not be implemented as a causal model.
Increasing $H$ gives delayed prefixes in Eq.~\eqref{eq:exposed-mass} nonzero weight, whereas prefixes first reached after $H$ remain outside the objective.

Equation~\eqref{eq:ideal} describes clean-path matching, while Eq.~\eqref{eq:longdmd} defines its noisy-path counterpart used to motivate DMD. Appendix~\ref{app:dmd-estimator} states the conditions under which the score-difference estimator is unbiased for the noisy objective. The clean-path bounds in Section~\ref{sec:theory} apply when their stated KL conditions hold; extensions to later blocks use the conditions in Appendix~\ref{app:finite-depth-details}.
For the interactive world model, training uses an action plan specified before the matched path and supplied to the teacher over the same horizon; extending the theory to state-feedback actions requires a joint state--action reference distribution.

Long-horizon DMD increases training memory and compute because the student unrolls farther and both score networks process a longer path.
At fixed spatial resolution, dense bidirectional attention has quadratic attention-score computation in temporal sequence length; total memory and runtime also depend on the implementation.
For a fixed output length and inference configuration, LongForcing retains the student architecture and sampling procedure, adding no inference operations.

\subsection{Comparison of Long-Horizon Rollout Strategies}
\label{app:method-comparison}

Figure~\ref{fig:method-comparison} contrasts the supervision coverage and inference mechanisms of representative long-horizon rollout strategies.
The comparison distinguishes where teacher supervision is applied from mechanisms that modify inference-time conditioning.

\begin{figure}[H]
    \centering
    \includegraphics[width=\linewidth]{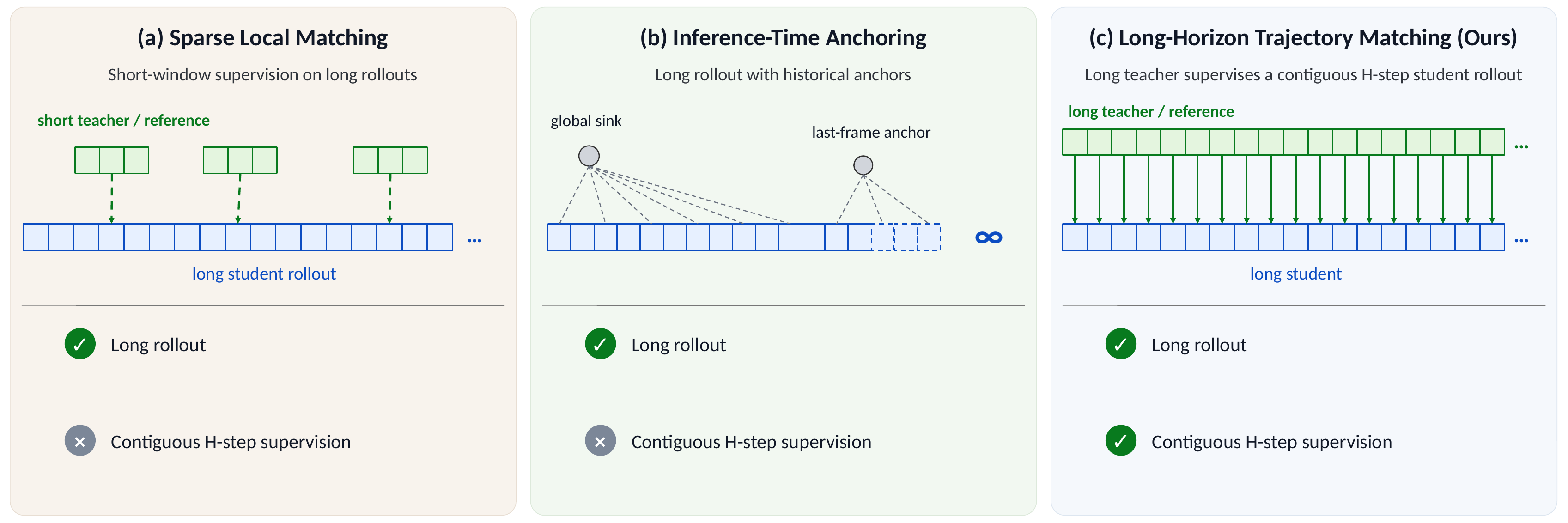}
    \caption{Mechanistic comparison of sparse local matching, inference-time anchoring, and LongForcing's contiguous $H$-step on-policy trajectory matching under an extended-horizon teacher.}
    \label{fig:method-comparison}
\end{figure}

\subsection{DMD Score-Difference Estimator}
\label{app:dmd-estimator}

For a shared condition $\mathcal C_H$, let $\mathbf z=G_\theta(\boldsymbol\xi;\mathcal C_H)$ be a clean student path and $\mathbf y=\alpha_s\mathbf z+\sigma_s\boldsymbol\epsilon$ its noised version. With $s_{\rm real}$ and $s_{\rm fake}$ denoting estimates of the teacher and student noisy scores, the unweighted, full-Jacobian score-difference estimator is
\begin{equation}
\label{eq:dmd}
\widehat{\mathbf g}_{\rm DMD}
=\alpha_s\left(\frac{\partial\mathbf z}{\partial\theta}\right)^{\!\top}
\operatorname{sg}\!\left[s_{\rm fake}(\mathbf y,s)-s_{\rm real}(\mathbf y,s)\right].
\end{equation}
Equation~\eqref{eq:dmd} presents the ideal, unweighted estimator; our implementation follows the sampling and weighting scheme of Causal Forcing~\citep{zhu2026causalforcing}.
Here $\operatorname{sg}$ treats the score difference as fixed during the generator update; condition arguments are suppressed. The factor $\alpha_s$ follows from differentiating the forward-noising map. For a condition and noise-level sampling law independent of $\theta$, exact scores, and full rollout differentiation, the expectation equals $\nabla_\theta\mathcal J_H^{\rm noisy}$, assuming differentiable positive noisy densities and valid interchange of differentiation and integration. In particular, the expected explicit derivative of the student log-density vanishes by normalization, leaving the pathwise score difference. Stop-gradient implements this expression without differentiating the score networks; score estimation, reweighting, normalization, or rollout-gradient truncation can introduce approximations to this ideal update.
Both score networks attend bidirectionally over the sampled path: the real-score network is frozen during DMD, whereas the fake-score network is updated online to track the student's rollout distribution.

\section{Theoretical Details and Proofs}
\label{app:theory}

\subsection{Formal Path Setup}
\label{app:path-setup}

All state, action, and latent spaces are taken to be standard Borel spaces, so the conditional path laws below exist. We use natural logarithms in KL and $d_{\mathrm{TV}}(P,Q)=\sup_A|P(A)-Q(A)|$.
Let $S_t=\Gamma(S_{t-1},\mathbf z_t)$ be the student's causal state after chunk $t$, including any finite context or cache required for the next prediction; fixed conditions and model parameters are suppressed in $\Gamma$.
All event comparisons use the same measurable path-to-state map and detector for student and teacher paths, with the model parameters held fixed.
For matched horizon $H$, define
$\mathcal C_H=(S_0,\mathbf a_{1:H},\mathbf c)$, where $\mathbf a_{1:H}$ is fixed before the window and $\mathbf c$ contains persistent conditions.
The causal student induces
\begin{equation}
Q_{\theta,H}(\mathbf z_{1:H}\mid\mathcal C_H)
=\prod_{t=1}^{H}p_\theta(\mathbf z_t\mid S_{t-1},\mathbf a_{1:t},\mathbf c),
\end{equation}
whereas the bidirectional teacher defines the joint reference $Q^{*,\mathrm{long}}_{\phi,H}(\mathbf z_{1:H}\mid\mathcal C_H)$ without requiring a causal implementation.
Fixing the actions open-loop prevents conditioning on actions generated from the same path; feedback policies instead require a joint state--action reference law.

For $\mathcal C_H\sim\mu_H$, the joint measures are $\overline Q_{\theta,H}=\mu_HQ_{\theta,H}$ and $\overline Q^*_{\phi,H}=\mu_HQ^{*,\mathrm{long}}_{\phi,H}$.
Because they share $\mu_H$, Eq.~\eqref{eq:ideal} is equivalently
$D_{\mathrm{KL}}(\overline Q_{\theta,H}\|\overline Q^*_{\phi,H})$.
For the exposure comparison, both events are evaluated on the same rollout law with the same initial-condition distribution and fixed detector. Identifying a shorter event with its extension to the longer path space gives $\mathcal B_{H_1}(\gamma)\subseteq\mathcal B_{H_2}(\gamma)$ for $H_1<H_2$. This construction does not require independently trained teachers at different horizons to have consistent marginals.

\subsection{Non-Identifiability and Delayed Exposure}
\label{app:proof-indistinguishable}

\begin{proof}[Proof of Proposition~\ref{prop:indistinguishable}]
For $b\in\{0,1\}$, define the uniform binary path law
\begin{equation}
Q_T^{(b)}=\operatorname{Unif}\!\left\{x_{1:T}\in\{0,1\}^{T}:
\sum_{t=1}^{T}x_t\equiv b\pmod 2\right\}.
\end{equation}
Thus $Q_T^{(0)}$ assigns mass to even-parity paths and $Q_T^{(1)}$ to odd-parity paths.
Both admit causal realizations: sample the first $T-1$ bits independently from $\operatorname{Bernoulli}(1/2)$, then set the last bit to enforce the prescribed parity.
For any coordinate set $I$ with $|I|=k<T$, fixing $x_I$ leaves exactly $2^{T-k-1}$ completions of either parity. Each full-path support has size $2^{T-1}$, hence
\begin{equation}
Q_T^{(0)}(x_I)=Q_T^{(1)}(x_I)=2^{-k}.
\end{equation}
In particular, every contiguous $W<T$ window has the same uniform marginal under both laws. Their full-path supports are disjoint, giving $d_{\mathrm{TV}}(Q_T^{(0)},Q_T^{(1)})=1$.
\end{proof}

\paragraph{Concrete example: three steps and two-step windows.}
For $T=3$ and $W=2$, the two laws are
\begin{equation}
\begin{aligned}
Q_3^{(0)}&=\operatorname{Unif}\{000,011,101,110\},\\
Q_3^{(1)}&=\operatorname{Unif}\{001,010,100,111\}.
\end{aligned}
\end{equation}
Under either law, both windows $(x_1,x_2)$ and $(x_2,x_3)$ assign probability $1/4$ to each of $00,01,10,11$. A discriminator restricted to either window therefore cannot distinguish the laws, while the full-path event $\{x_1+x_2+x_3\text{ is even}\}$ has probability one under $Q_3^{(0)}$ and zero under $Q_3^{(1)}$.
The missing information is the coupling across windows: local agreement alone leaves a global path constraint undetermined.

The proposition concerns objectives defined by local-window marginals; exact matching of every causal conditional on its full support would determine the joint path.
For a measurable region $\mathcal R$, let $\tau_{\mathcal R}=\inf\{t:S_t\in\mathcal R\}$ and $h_{\mathcal R}(H)=\Pr(\tau_{\mathcal R}\le H)$.
With $N$ independent horizon-$H$ rollouts,
\begin{equation}
\Pr(\mathcal R\text{ is never exposed})=(1-h_{\mathcal R}(H))^N.
\end{equation}
If $h_{\mathcal R}(H_1)=0$, more horizon-$H_1$ samples cannot expose $\mathcal R$; sequential depth cannot be replaced by sample count.
This is an exposure statement for a fixed rollout distribution, not a monotonic improvement claim across separately trained students.

\subsection{Within-Horizon Path Control}
\label{app:proof-within}

\begin{proof}[Proof of Theorem~\ref{thm:within}]
Conditionally on $\mathcal C_H$, Pinsker's inequality gives
\begin{equation}
d_{\mathrm{TV}}\!\left(Q_{\theta,H},Q^{*,\mathrm{long}}_{\phi,H}\right)
\le\sqrt{\tfrac12 D_{\mathrm{KL}}\!\left(Q_{\theta,H}\|Q^{*,\mathrm{long}}_{\phi,H}\right)}.
\end{equation}
Average over $\mu_H$ and apply Jensen's inequality to obtain
$d_{\mathrm{TV}}(\overline Q_{\theta,H},\overline Q^*_{\phi,H})\le\sqrt{\mathcal L_H^{\mathrm{ideal}}/2}$.
The event bound follows from the definition of total variation.
\end{proof}

The same argument controls any measurable $F_H\in[0,1]$:
\begin{equation}
\left|\mathbb E_{\overline Q_{\theta,H}}F_H-\mathbb E_{\overline Q^*_{\phi,H}}F_H\right|
\le\sqrt{\mathcal L_H^{\mathrm{ideal}}/2}.
\end{equation}
Taking $F_H=H^{-1}\sum_{t=1}^H\mathbf1\{S_t\in\mathcal R\}$ yields the finite-horizon occupancy bound without assuming ergodicity.
Indeed, $F_H=\int_0^1\mathbf1\{F_H>u\}\,du$; integrating the event bound for $\{F_H>u\}$ proves the bounded-functional inequality above.
The next subsection extends the analysis to later block boundaries under explicit leakage and recovery conditions.

\subsection{Finite-Depth Drift Control}
\label{app:finite-depth-details}
\label{app:proof-recursion}

For an open-loop action block, define the complete condition
\begin{equation}
X_j=(S_0,\mathbf z_{1:jB},\mathbf a_{1:(j+1)B},\mathbf c),
\qquad \nu_j=\operatorname{Law}_{Q_\theta}(X_j),
\end{equation}
where the future suffix $\mathbf a_{jB+1:(j+1)B}$ is fixed before the block.
If the block is a conditional of a matched $H$-step joint path, use $X_j=(\mathcal C_H,\mathbf z_{1:jB})$ instead, retaining the entire teacher condition, including actions after the block.
The full prefix may be replaced only by an augmented state sufficient for both student and teacher path laws.
Let $\rho_j^{\mathcal G}=\nu_j(\cdot\mid S_{jB}\in\mathcal G)$ and $\rho_j^{\mathcal D}=\nu_j(\cdot\mid S_{jB}\in\mathcal D)$ when the corresponding events have positive mass; otherwise their conditional values are immaterial in Eq.~\eqref{eq:recursion-exact}.

For each start $x$, the teacher reference is a specified measurable continuation kernel conditioned on that same prefix and action block. Within a matched joint path, it is the corresponding chain-rule conditional; a continuation outside that horizon requires a separately specified reference.
Let $\lambda_j^*$ be the teacher's average terminal probability of $\mathcal D$ from $\rho_j^{\mathcal G}$ and $\eta_j^*$ its average terminal probability of $\mathcal G$ from $\rho_j^{\mathcal D}$.
Define
\begin{equation}
\epsilon_j^{\mathcal G}=\mathbb E_{x\sim\rho_j^{\mathcal G}}
D_{\mathrm{KL}}\!\left(Q_{\theta,B}(\cdot\mid x)\,\middle\|\,Q^{*,\mathrm{long}}_{\phi,B}(\cdot\mid x)\right),
\end{equation}
and define $\epsilon_j^{\mathcal D}$ analogously under $\rho_j^{\mathcal D}$.
The finite-KL assumptions entail student-to-teacher absolute continuity almost everywhere under the respective start distribution.
Write $P_j^x$ and $R_j^x$ for these student and teacher block laws, and $K_j(x)=D_{\mathrm{KL}}(P_j^x\|R_j^x)$.
For either terminal region $E\in\{\mathcal G,\mathcal D\}$, its probability difference is at most $d_{\mathrm{TV}}(P_j^x,R_j^x)\le\sqrt{K_j(x)/2}$, since the terminal state is a measurable function of the start and block path.
Averaging under the appropriate conditional start law and applying Jensen's inequality gives
\begin{equation}
\begin{aligned}
a_j-\lambda_j^*&\le\mathbb E_{\rho_j^{\mathcal G}}\sqrt{K_j/2}
\le\sqrt{\epsilon_j^{\mathcal G}/2},\\
\eta_j^*-\eta_j&\le\mathbb E_{\rho_j^{\mathcal D}}\sqrt{K_j/2}
\le\sqrt{\epsilon_j^{\mathcal D}/2}.
\end{aligned}
\end{equation}
Combining the second inequality with $\eta_j\ge0$ proves Eq.~\eqref{eq:block-control}.

\paragraph{Relation to the path objective.}
Let $\mathcal L_{B,j}=\mathbb E_{\nu_j}K_j$ denote the block-specific on-policy KL. Splitting its expectation by the start region gives
\begin{equation}
\label{eq:block-kl-decomposition}
\mathcal L_{B,j}=(1-b_j)\epsilon_j^{\mathcal G}+b_j\epsilon_j^{\mathcal D}.
\end{equation}
For $0<b_j<1$, nonnegativity therefore implies $\epsilon_j^{\mathcal G}\le\mathcal L_{B,j}/(1-b_j)$ and $\epsilon_j^{\mathcal D}\le\mathcal L_{B,j}/b_j$; a zero-mass term contributes zero.
When all blocks use the chain-rule conditionals of the same matched joint laws and $H=JB$, the KL chain rule gives $\mathcal L_H^{\mathrm{ideal}}=\sum_{j=0}^{J-1}\mathcal L_{B,j}$.
For $JB<H$, the remaining suffix contributes an additional nonnegative conditional KL.
If instead a block loss is averaged under a different training start law $\widetilde\mu_j$, transfer to $\rho_j^E$ requires a measure comparison: whenever $\rho_j^E\ll\widetilde\mu_j$ and $d\rho_j^E/d\widetilde\mu_j\le M_j^E$, integration gives $\epsilon_j^E\le M_j^E\mathbb E_{\widetilde\mu_j}K_j$.
These relations control conditional averages; the finite-depth proposition explicitly assumes the required good- and degraded-start bounds at each evaluated block.

\begin{proof}[Proof of Proposition~\ref{thm:recursion}]
Under the proposition's assumptions, Eq.~\eqref{eq:block-control} gives $a_j\le A$ and $\eta_j\ge\underline\eta$. Together with Eq.~\eqref{eq:recursion-exact}, this implies
$b_{j+1}\le A(1-b_j)+(1-\underline\eta)b_j\le A+(1-\underline\eta)b_j$.
Probability bounds are taken in $[0,1]$, so $r=1-\underline\eta\in[0,1)$.
Induction yields $b_j\le r^jb_0+A\sum_{k=0}^{j-1}r^k=r^jb_0+A(1-r^j)/\underline\eta$ for $1\le j\le J$.
Together with the case $j=0$ and $b_j\le1$, this proves Eq.~\eqref{eq:recursion-bound}.
The monotone relaxation avoids any sign requirement on $1-A-\underline\eta$.
\end{proof}

\paragraph{Partial start coverage.}
Coverage is measured under the relevant conditional start law. If the student's per-start leakage is at most $A$ and its per-start recovery is at least $\underline\eta$ only on subsets with masses at least $1-\delta_j^{\mathcal G}$ under $\rho_j^{\mathcal G}$ and $1-\delta_j^{\mathcal D}$ under $\rho_j^{\mathcal D}$, respectively, the averaged bounds become
\begin{equation}
 a_j\le A+(1-A)\delta_j^{\mathcal G},\qquad
 \eta_j\ge(1-\delta_j^{\mathcal D})\underline\eta.
\end{equation}
To see this, bound leakage on uncovered good starts by one and recovery on uncovered degraded starts by zero, then average over the two subsets.
These adjusted quantities enter the recursion at each block. With complete degraded-start coverage and $\delta_j^{\mathcal G}\le\delta$, each affine update acquires at most $(1-A)\delta$; summing its geometrically weighted contributions gives additional envelope slack at block $j$ of at most $(1-A)\delta[1-(1-\underline\eta)^j]/\underline\eta$.

\paragraph{Feedback actions.}
The total-probability identity in Eq.~\eqref{eq:recursion-exact} also holds for feedback policies, without Markov or stationarity assumptions.
For its KL-based bounds, specify a policy $\pi_t(d\mathbf a_t\mid h_{t-1})$ and a nonanticipating teacher continuation $q_t^*(d\mathbf z_t\mid h_{t-1},\mathbf a_t)$, where $h_{t-1}$ contains the available history and policy state.
From the same block start, define joint action--latent laws
\begin{equation}
\begin{aligned}
P_j^\pi&=\prod_{t=jB+1}^{(j+1)B}\pi_t(d\mathbf a_t\mid h_{t-1})\,
p_\theta(d\mathbf z_t\mid h_{t-1},\mathbf a_t),\\
R_j^\pi&=\prod_{t=jB+1}^{(j+1)B}\pi_t(d\mathbf a_t\mid h_{t-1})\,
q_t^*(d\mathbf z_t\mid h_{t-1},\mathbf a_t).
\end{aligned}
\end{equation}
Under finite KL, the shared policy factors cancel in the likelihood ratio, and the chain rule expresses $D_{\mathrm{KL}}(P_j^\pi\|R_j^\pi)$ as the sum of conditional latent KLs averaged over student histories and actions.
Applying the preceding Pinsker argument to these joint laws proves the feedback version with its own KL, leakage, and recovery quantities.
This construction requires a specified nonanticipating reference; conditioning a bidirectional open-loop teacher on actions realized through feedback does not supply that reference automatically.
Throughout, $b_j$ measures occupancy at boundaries spaced by the fixed stride $B$; within-block visits are covered by path events or the bounded functionals in Appendix~\ref{app:proof-within}.

\subsection{Conditional Long-Run Extension}
\label{app:long-run}

Fix a persistent condition $\mathbf c$ and a time-homogeneous stationary policy $\pi$, and let $Y_j$ include the generator cache and policy state sufficient for a time-homogeneous Markov boundary process.
Suppose its initial law is supported on a measurable class $\mathcal R_{c,\pi}$, with $P_{\theta,B}^{c,\pi}(y,\mathcal R_{c,\pi})=1$ for every $y$ in that class, and its block kernel satisfies the Doeblin minorization
$P_{\theta,B}^{c,\pi}(y,\cdot)\ge\beta_{c,\pi}m_{c,\pi}(\cdot)$ there, with $\beta_{c,\pi}\in(0,1]$ and $m_{c,\pi}$ a probability measure supported on the class.
The regions $\mathcal G$ and $\mathcal D$ are lifted to $Y_j$ through its visual-state component.
Then the process for this fixed condition and policy has a unique stationary distribution $\nu_\infty^{c,\pi}$, and its empirical degraded-boundary fraction converges almost surely to $\nu_\infty^{c,\pi}(\mathcal D)$.
If, pointwise on $\mathcal R_{c,\pi}$, good states leak with probability at most $A_{c,\pi}$ and degraded states return with probability at least $\eta_{c,\pi}>0$, then
\begin{equation}
    \nu_\infty^{c,\pi}(\mathcal D)
    \le\frac{A_{c,\pi}}{A_{c,\pi}+\eta_{c,\pi}}.
\end{equation}
\begin{proof}
Suppress $c,\pi$ and write $P=P_{\theta,B}^{c,\pi}$. For $0<\beta<1$, the residual kernel $R=(P-\beta m)/(1-\beta)$ is a probability kernel on the closed class.
Thus $P=\beta m+(1-\beta)R$, and
\begin{equation}
\nu_\infty=\beta\sum_{k=0}^{\infty}(1-\beta)^k mR^k
\end{equation}
is a probability measure satisfying $\nu_\infty P=\nu_\infty$ by shifting the series.
Since Markov kernels contract total variation,
\begin{equation}
d_{\mathrm{TV}}(\nu P^k,\nu'P^k)
\le(1-\beta)^k d_{\mathrm{TV}}(\nu,\nu').
\end{equation}
This proves uniqueness and uniform convergence to $\nu_\infty$.
The decomposition also constructs regeneration times: at each transition, independently refresh from $m$ with probability $\beta$, otherwise use $R$.
After the first refresh, cycles are independent and identically distributed, with geometric lengths of mean $1/\beta$ and bounded per-step rewards.
The renewal-reward strong law gives almost-sure convergence of the empirical average of $\mathbf1_{\mathcal D}$; its limit equals $\nu_\infty(\mathcal D)$ by the displayed uniform convergence and boundedness.
For $\beta=1$, the process is i.i.d.\ from $m$ after one transition, giving the same conclusions.
Finally, stationarity and the pointwise leakage/return bounds give
$b_\infty\le A_{c,\pi}(1-b_\infty)+(1-\eta_{c,\pi})b_\infty$; rearranging proves the stated fraction bound.
\end{proof}
This is the standard minorization and regeneration argument for general-state-space chains~\citep{roberts2004general}.
This extension assumes the stated uniform minorization, leakage, and recovery conditions for the fixed prompt and policy. It bounds the long-run fraction of degraded block boundaries, rather than the probability of ever visiting a degraded state.

\section{Evaluation Details and Additional Results}

\subsection{Full VBench Results}
\label{app:full-vbench-results}

\begin{table}[H]
\centering
\caption{\textbf{Full 5-second VBench comparison.}}
\label{tab:full-main-t2v}
\small
\resizebox{\textwidth}{!}{%
\begin{tabular}{@{}lccccccc@{}}
\toprule
Model & \#Params & Resolution & \shortstack{Throughput\\(FPS) $\uparrow$} & \shortstack{Latency\\(s) $\downarrow$} & \shortstack{VBench\\Total $\uparrow$} & Quality $\uparrow$ & Semantic $\uparrow$ \\
\midrule
\multicolumn{8}{@{}l}{\emph{Bidirectional video diffusion models}} \\
LTX-Video~\citep{hacohen2024ltx} & 1.9B & $768{\times}512$ & 8.98 & 13.5 & 79.83 & 81.88 & 71.62 \\
Wan2.1~\citep{wan2025wan} & 1.3B & $832{\times}480$ & 0.78 & 103 & 83.37 & 84.30 & 79.65 \\
\midrule
\multicolumn{8}{@{}l}{\emph{Autoregressive video diffusion models}} \\
NOVA~\citep{deng2025nova} & 0.6B & $768{\times}480$ & 0.88 & 4.1 & 80.31 & 80.66 & 78.92 \\
Pyramid Flow~\citep{jin2024pyramidflow} & 2B & $640{\times}384$ & 6.70 & 2.5 & 80.75 & 83.41 & 70.11 \\
SkyReels-V2~\citep{chen2025skyreelsv2} & 1.3B & $960{\times}540$ & 0.49 & 112 & 81.97 & 83.96 & 74.01 \\
MAGI-1~\citep{teng2025magi} & 4.5B & $832{\times}480$ & 0.19 & 282 & 78.88 & 81.67 & 67.72 \\
\midrule
\multicolumn{8}{@{}l}{\emph{Distilled autoregressive video models}} \\
CausVid~\citep{yin2025causvid} & 1.3B & $832{\times}480$ & 17.0 & 0.69 & 81.33 & 83.98 & 70.72 \\
Self Forcing~\citep{huang2025selfforcing} & 1.3B & $832{\times}480$ & 17.0 & 0.69 & 83.74 & 84.48 & \underline{80.77} \\
Causal Forcing~\citep{zhu2026causalforcing} & 1.3B & $832{\times}480$ & 17.0 & 0.69 & \underline{84.04} & \underline{84.59} & \textbf{81.84} \\
\textbf{LongForcing (ours)} & 1.3B & $832{\times}480$ & $17.0^{\dagger}$ & $0.69^{\dagger}$ & \textbf{84.56} & \textbf{85.52} & 80.72 \\
\bottomrule
\end{tabular}
}
\par\smallskip
\begin{minipage}{\textwidth}
\scriptsize $\dagger$ Reference throughput and latency quoted from Causal Forcing~\citep{zhu2026causalforcing} for its inference configuration, which LongForcing retains.
\end{minipage}
\end{table}

\begin{table}[H]
\centering
\caption{\textbf{Full 30-second VBench-Long comparison.}}
\label{tab:full-main-t2v-long}
\setlength{\tabcolsep}{3.2pt}
\scriptsize
\resizebox{\textwidth}{!}{%
\begin{tabular}{@{}lccccccccc@{}}
\toprule
\multirow{2}{*}{Model}
& \multicolumn{3}{c}{Aggregate}
& \multicolumn{2}{c}{Long-horizon consistency}
& \multicolumn{2}{c}{Visual quality}
& \multicolumn{2}{c}{Temporal evolution} \\
\cmidrule(lr){2-4}\cmidrule(lr){5-6}\cmidrule(lr){7-8}\cmidrule(l){9-10}
& Total & Quality & Semantic
& \shortstack{Subject\\Cons.} & \shortstack{Background\\Cons.}
& Aesthetic & Imaging
& Dynamic & \shortstack{Temporal\\Style} \\
\midrule
LongLive~\citep{yang2025longlive}
& \underline{82.80} & 83.21 & \textbf{81.17}
& \textbf{98.36} & \underline{96.25}
& \textbf{63.77} & \textbf{68.35}
& 38.89 & \textbf{24.15} \\
Self Forcing~\citep{huang2025selfforcing}
& 81.49 & 82.67 & 76.77
& \underline{97.76} & 95.00
& 61.72 & \underline{67.70}
& 41.93 & 23.56 \\
Causal Forcing~\citep{zhu2026causalforcing}
& 81.05 & \underline{83.44} & 71.52
& 96.24 & 95.59
& 58.85 & 66.92
& \textbf{75.46} & 22.61 \\
\midrule
\textbf{LongForcing (ours)}
& \textbf{82.83} & \textbf{83.99} & \underline{78.19}
& 97.40 & \textbf{96.58}
& \underline{62.87} & 64.19
& \underline{60.93} & \underline{23.61} \\
\bottomrule
\end{tabular}
}
\vspace{1pt}
\begin{minipage}{0.99\textwidth}
\scriptsize Best and second-best results are shown in \textbf{bold} and \underline{underlined}, respectively. All scores are percentages.
\end{minipage}
\end{table}

\subsection{Causal-VAE Drift Statistics}
\label{app:drift-metric}

We re-encode 20 videos per model with the shared causal VAE.
Let $m_{v,c,t}$ be the spatial mean of channel $c$ at zero-based latent index $t$ in video $v$, and define its early reference as
\begin{equation}
    \bar m_{v,c}=\frac{1}{20}\sum_{r=1}^{20}m_{v,c,r}.
\end{equation}
This reference excludes the first latent ($t=0$) and averages latent indices $1$--$20$, corresponding to zero-based RGB frame indices $1$--$80$ (80 frames).
Panel (b) of Figure~\ref{fig:latent-drift} plots $\operatorname{median}_{v,c}|m_{v,c,t}-\bar m_{v,c}|$ at each $t$.
For the 81-frame model, panel (c) plots the signed channel-wise statistic $\operatorname{median}_{v}(m_{v,c,t}-\bar m_{v,c})$.
These statistics summarize the magnitude and channel-wise direction of latent shifts relative to the early reference window.

\subsection{Time-Resolved Drift Score}
\label{app:drift-score}

We sample frames at 1\,Hz from 16-FPS videos, using indices $0,16,\ldots,944$ (60 observations per video). High saturation and RGB clipping are measured at decoded resolution. With OpenCV HSV saturation and value normalized to $[0,1]$, the high-saturation ratio is the fraction of all pixels satisfying $S>0.85$ and $V>0.15$. The RGB-clipping ratio is the fraction with any 8-bit RGB channel at most 5 or at least 250.
For chroma edges, frames are resized to a 256-pixel short side with aspect ratio preserved, integer-rounded dimensions, and OpenCV area interpolation. After RGB-to-YCrCb conversion, let $g_Y$, $g_{Cr}$, and $g_{Cb}$ be the Euclidean magnitudes of horizontal and vertical $3{\times}3$ Sobel responses, computed in floating point on the unnormalized 8-bit channel scale. With $g_C=\sqrt{g_{Cr}^2+g_{Cb}^2}$, the chroma-edge ratio is the fraction of resized pixels satisfying $g_C>18$ and $g_C>1.5(g_Y+1)$.
For text-to-video HPSv3~\citep{ma2025hpsv3}, each matched pair uses identical filename-derived prompt prefixes, held fixed across models and timestamps. These scores measure temporal changes relative to each video's early reference. Interactive HPSv3 uses the prompt associated with each scene ID.

For video $i$ and metric $m$, let $b_{im}$ be the mean over seconds 0--4.
We orient deterioration as $d_{im}(t)=b_{im}-x_{im}(t)$ for HPSv3 and $d_{im}(t)=x_{im}(t)-b_{im}$ for three artifact ratios: high saturation, RGB clipping, and chroma edges.
The normalized deterioration is
\begin{equation}
z_{im}(t)=\operatorname{clip}\!\left(\frac{[d_{im}(t)]_+}{q_m},0,1\right).
\end{equation}
The frozen scales are the pooled 95th percentiles of $[d_{im}(t)]_+$, including zeros, over both the 81- and 321-frame text-to-video groups, all 20 matched videos per group, all 60 sampled seconds, and each metric independently:
$q_m=(12.5961,0.2977,0.4786,0.3202)$ for HPSv3, high saturation, RGB clipping, and chroma edges, respectively.
The per-video score is
\begin{equation}
\label{eq:drift-score}
D_i(t)=100\left[\tfrac12 z_{i,\mathrm{HPS}}(t)+\tfrac16\!\sum_{m\in\{\mathrm{sat},\mathrm{clip},\mathrm{chroma}\}}z_{im}(t)\right].
\end{equation}
Thus HPSv3 and the correlated artifact family receive equal total weight.
The score measures deterioration relative to each video's early reference: lower values indicate smaller increases in the selected degradation indicators. Absolute perceptual quality and motion are assessed separately by the benchmark and human evaluations.

Figure~\ref{fig:full-drift-score} applies the same frozen scales to both regimes.
Curves use a centered 5-second moving average, truncated to the available observations at each endpoint; numerical endpoint comparisons below use the unsmoothed 59-second observations.
The pointwise 95\% percentile bootstrap intervals resample complete prompt or scene trajectories 2,000 times  and apply the same smoothing within each replicate, keeping the normalization scales fixed. At each timestamp, the interval endpoints are the 2.5th and 97.5th percentiles of the bootstrap means.
The text-to-video and interactive panels contain 20 prompt-matched and 23 scene-matched videos per model, respectively.

\begin{figure}[H]
    \centering
    \includegraphics[width=\textwidth]{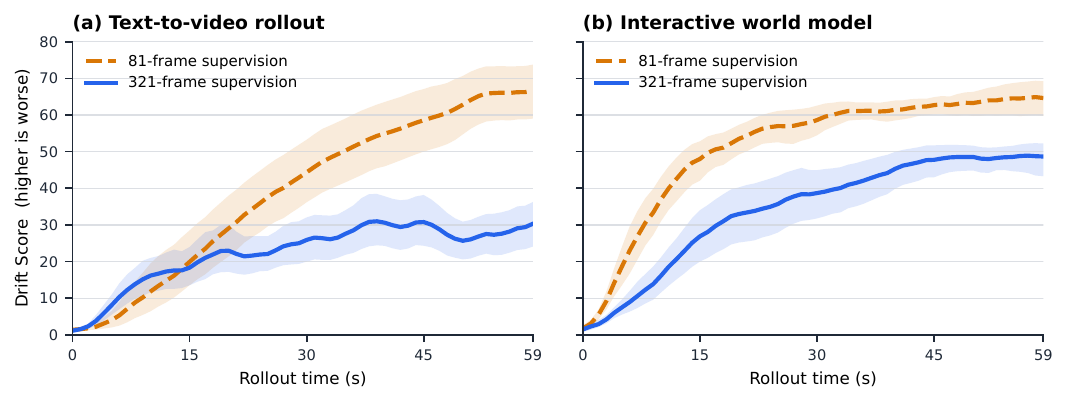}
    \caption{\textbf{Full time-resolved drift comparison.}
    Curves show mean Drift Score with 95\% bootstrap confidence intervals for text-to-video and interactive rollouts.}
    \label{fig:full-drift-score}
\end{figure}

At 59 seconds, 321-frame supervision reduces the mean Drift Score from 67.46 to 31.91 for text-to-video generation and from 64.76 to 48.95 for interactive world modeling, corresponding to relative reductions of 52.7\% and 24.4\%, respectively.

\subsection{Human Preference Evaluation}
\label{app:human-study}
Our analysis covers 20 shared cases evaluated by all 20 participants: 10 text-to-video prompts and 10 interactive world-model scenes.
Each pair compares 81- and 321-frame supervision with model identities hidden and left/right placement randomized.
Evaluators select A better, Tie, or B better for overall preference, temporal coherence, visual quality, and motion naturalness.
Each participant provides one response per case and criterion, yielding $20\times20=400$ judgments per criterion and 1,600 judgments across the four criteria.
Figure~\ref{fig:drift-human-study} reports descriptive preference proportions, assigning equal weight to each evaluator--case response and retaining ties in the denominator.

\section{Additional Qualitative Results}
\label{app:qualitative}

\subsection{Text-to-Video Rollouts}

Figure~\ref{fig:qualitative} follows four 60-second text-to-video rollouts trained with 321-frame supervision.
Across portrait, human-motion, animated-aerial, and deep-space prompts, the displayed frames show recognizable subjects and scene style as content evolves beyond the supervised horizon.

\begin{figure}[H]
    \centering
    \includegraphics[width=\textwidth]{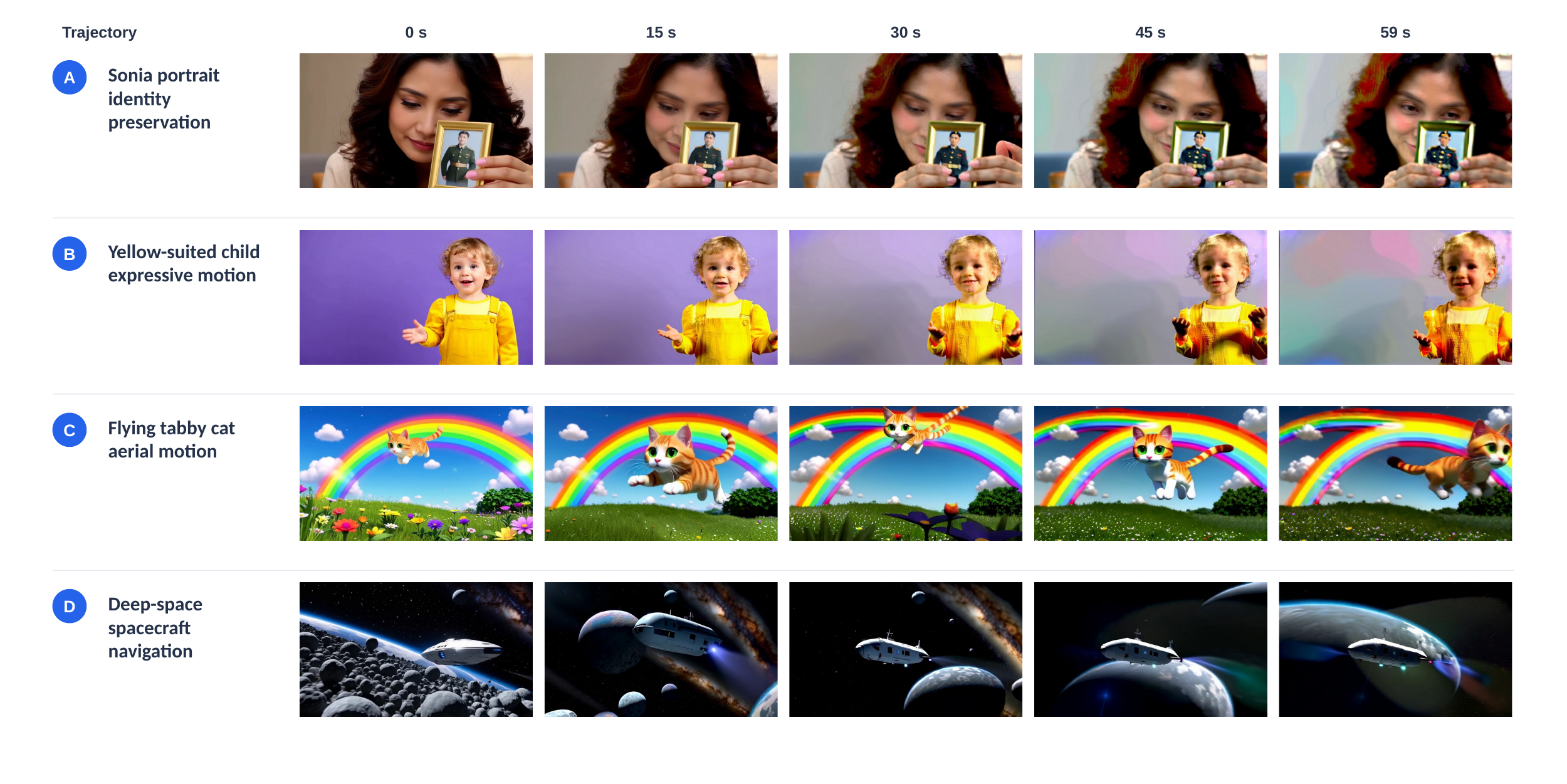}
    \caption{Qualitative 60-second text-to-video rollouts. Each row follows one LongForcing trajectory trained with 321-frame supervision; columns show frames at fixed video timestamps.}
    \label{fig:qualitative}
\end{figure}

\subsection{Supervision-Horizon Comparison}

Figure~\ref{fig:qualitative-horizon} compares matched prompts or initial scenes.
The 81-frame text-to-video examples exhibit progressive saturation and structural distortion, while the interactive examples develop horizontal banding and geometry collapse.
With 321-frame supervision, the displayed snapshots from all four trajectories preserve more scene structure and do not show the same late-stage artifacts.

\begin{figure}[H]
    \centering
    \includegraphics[width=\textwidth]{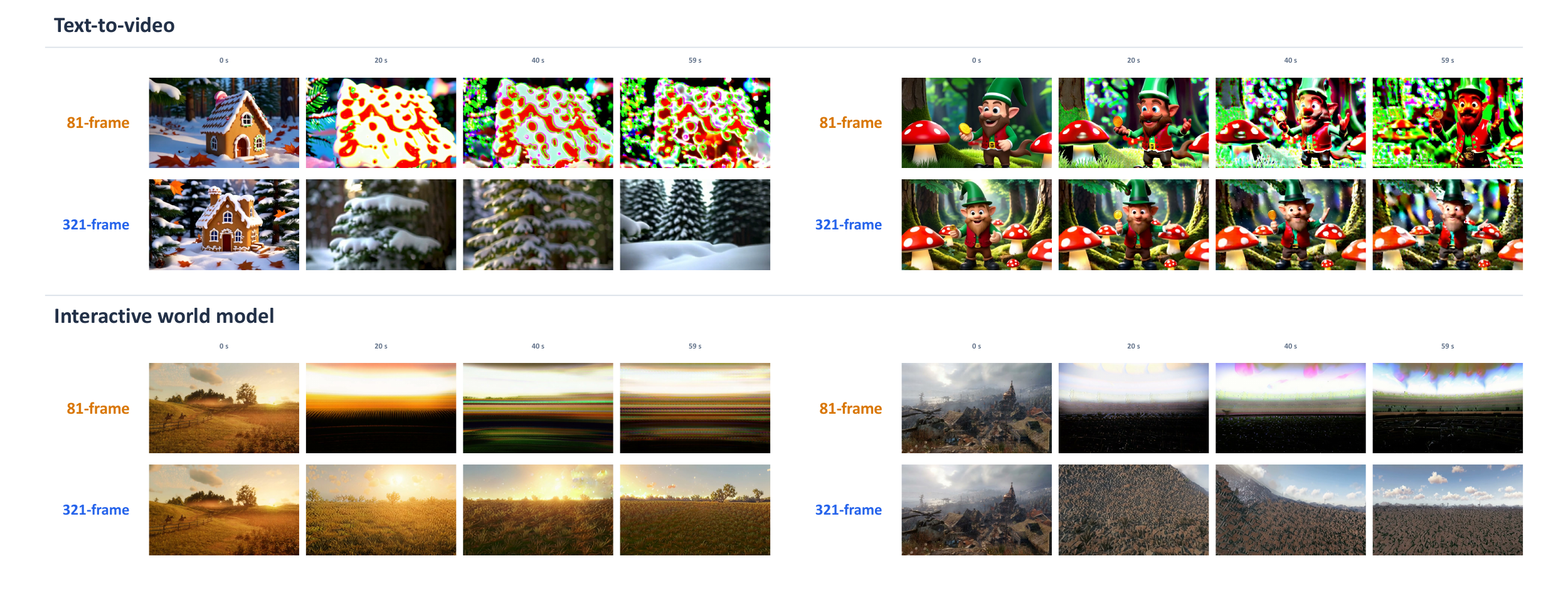}
    \caption{Qualitative comparison across supervision horizons. Upper and lower rows use 81- and 321-frame supervision, respectively, at matched rollout times from 0 to 59 seconds.}
    \label{fig:qualitative-horizon}
\end{figure}

\section{Discussion}
\label{app:discussion}

\paragraph{When does LongForcing help?}
Horizon scaling targets failures first reached beyond the shorter training window. Extending the on-policy rollout makes these later prefixes available to the teacher, while the horizon ablations quantify the resulting empirical changes in generation quality. The benefit depends on both supervision coverage and the quality of the teacher signal on those prefixes.

\paragraph{Cost.}
The additional computation occurs during training: the student generates longer paths and both score networks process longer sequences. At fixed spatial resolution, dense bidirectional attention has quadratic attention-score computation in temporal length; practical memory and runtime depend on the attention and gradient implementations. For a fixed inference configuration and output length, the student uses the same architecture and sampling operations as the short-horizon baseline.

\paragraph{Scope.}
The finite-horizon results quantify teacher-relative path agreement, while the blockwise extension additionally assumes controlled leakage, mismatch, and recovery on the student's start distribution. The long-run extension states its uniform dynamical assumptions separately. These conditions concern the specified model, domain, and action protocol. Visual quality, motion, and action following are evaluated separately, complementing the drift analysis.

\paragraph{How long is enough?}
A useful supervision horizon depends on when task-relevant dependencies and failure modes become observable, rather than on a universal duration.
For scene roaming, delayed effects may involve revisiting a location or maintaining an object after a long absence; for goal-directed interaction, an action's consequence may appear much later.
Our horizon sweep measures the benefit of extending supervision within the tested range; determining a sufficient horizon for a new task calls for evaluating those dependencies over the intended deployment duration.
A promising direction is to allocate longer training rollouts when delayed errors remain frequent, and to compare this adaptive allocation with fixed horizons under matched compute budgets.
Teacher quality and persistent memory offer complementary ways to make additional temporal coverage useful.

\paragraph{From visual rollout to long-term planning.}
Maintaining a coherent visual trajectory and deciding which events should occur over a long story or task are complementary problems.
LongForcing strengthens the former by supervising longer self-generated paths; long-term planning additionally requires representing goals, organizing events, and retaining consequences across distant steps.
This suggests a hierarchical approach in which a planner maintains an abstract state and selects subgoals, persistent memory records relevant past events, and a video model renders the resulting interactions.
Long-horizon supervision could then assess whether the rendered trajectory respects those plans and remembered consequences.
Evaluating delayed action effects, goal completion, and narrative consistency alongside visual quality would help determine whether longer rollouts also support more coherent long-term behavior.

%% file: main.bbl
\begin{thebibliography}{10}

\bibitem{bruce2024genie}
Jake Bruce, Michael~D. Dennis, Ashley Edwards, Jack Parker-Holder, Yuge Shi,
  Edward Hughes, Matthew Lai, Aditi Mavalankar, Richie Steigerwald, Chris Apps,
  et~al.
\newblock {Genie}: Generative interactive environments.
\newblock In {\em Proceedings of the 41st International Conference on Machine
  Learning}, volume 235 of {\em Proceedings of Machine Learning Research},
  pages 4603--4623. PMLR, 2024.

\bibitem{google2025genie3}
{Google DeepMind}.
\newblock {Genie} 3.
\newblock Google DeepMind, 2025.

\bibitem{valevski2025gamengen}
Dani Valevski, Yaniv Leviathan, Moab Arar, and Shlomi Fruchter.
\newblock Diffusion models are real-time game engines.
\newblock In {\em International Conference on Learning Representations}, 2025.

\bibitem{decart2024oasis}
{Decart}, Julian Quevedo, Quinn McIntyre, Spruce Campbell, Xinlei Chen, and
  Robert Wachen.
\newblock {Oasis}: A universe in a transformer.
\newblock Decart and Etched, 2024.

\bibitem{zhang2025matrixgame}
Yifan Zhang, Chunli Peng, Boyang Wang, Puyi Wang, Qingcheng Zhu, Fei Kang, Biao
  Jiang, Zedong Gao, Eric Li, Yang Liu, et~al.
\newblock {Matrix-Game}: Interactive world foundation model.
\newblock {\em arXiv preprint arXiv:2506.18701}, 2025.

\bibitem{sun2025worldplay}
Wenqiang Sun, Haiyu Zhang, Haoyuan Wang, Junta Wu, Zehan Wang, Zhenwei Wang,
  Yunhong Wang, Jun Zhang, Tengfei Wang, and Chunchao Guo.
\newblock {WorldPlay}: Towards long-term geometric consistency for real-time
  interactive world modeling.
\newblock In {\em International Conference on Machine Learning}, 2026.

\bibitem{kim2024fifo}
Jihwan Kim, Junoh Kang, Jinyoung Choi, and Bohyung Han.
\newblock {FIFO-Diffusion}: Generating infinite videos from text without
  training.
\newblock In {\em Advances in Neural Information Processing Systems},
  volume~37, 2024.

\bibitem{huang2025selfforcing}
Xun Huang, Zhengqi Li, Guande He, Mingyuan Zhou, and Eli Shechtman.
\newblock {Self Forcing}: Bridging the train-test gap in autoregressive video
  diffusion.
\newblock In {\em Advances in Neural Information Processing Systems},
  volume~38, 2025.

\bibitem{yin2025causvid}
Tianwei Yin, Qiang Zhang, Richard Zhang, William~T. Freeman, Fr{\'e}do Durand,
  Eli Shechtman, and Xun Huang.
\newblock From slow bidirectional to fast autoregressive video diffusion
  models.
\newblock In {\em Proceedings of the IEEE/CVF Conference on Computer Vision and
  Pattern Recognition}, pages 22963--22974, 2025.

\bibitem{zhu2026causalforcing}
Hongzhou Zhu, Min Zhao, Guande He, Hang Su, Chongxuan Li, and Jun Zhu.
\newblock {Causal Forcing}: Autoregressive diffusion distillation done right
  for high-quality real-time interactive video generation.
\newblock In {\em International Conference on Machine Learning}, 2026.
\newblock arXiv:2602.02214.

\bibitem{cui2025selfforcingpp}
Jiaxing Cui, Jie Wu, Ming Li, Tao Yang, Xiaojie Li, Rui Wang, Andrew Bai,
  Yuanhao Ban, and Cho-Jui Hsieh.
\newblock {Self-Forcing++}: Towards minute-scale high-quality video generation.
\newblock In {\em International Conference on Learning Representations}, 2026.

\bibitem{li2026rolling}
Haodong Li, Shaoteng Liu, Zhe Lin, and Manmohan Chandraker.
\newblock {Rolling Sink}: Bridging limited-horizon training and open-ended
  testing in autoregressive video diffusion.
\newblock {\em arXiv preprint arXiv:2602.07775}, 2026.

\bibitem{yi2025deepforcing}
Jung Yi, Wooseok Jang, Paul~Hyunbin Cho, Jisu Nam, Heeji Yoon, and Seungryong
  Kim.
\newblock {Deep Forcing}: Training-free long video generation with deep sink
  and participative compression.
\newblock In {\em International Conference on Machine Learning}, 2026.

\bibitem{cui2026lol}
Justin Cui, Jie Wu, Ming Li, Tao Yang, Xiaojie Li, Rui Wang, Andrew Bai,
  Yuanhao Ban, and Cho-Jui Hsieh.
\newblock {LoL}: Longer than longer, scaling video generation to hour.
\newblock In {\em Proceedings of the IEEE/CVF Conference on Computer Vision and
  Pattern Recognition}, pages 38132--38142, 2026.

\bibitem{zhu2025astra}
Yixuan Zhu, Jiaqi Feng, Wenzhao Zheng, Yuan Gao, Xin Tao, Pengfei Wan, Jiwen
  Lu, and Jie Zhou.
\newblock {Astra}: General interactive world model with autoregressive
  denoising.
\newblock In {\em International Conference on Learning Representations}, 2026.

\bibitem{wu2021godiva}
Chenfei Wu, Lun Huang, Qianxi Zhang, Binyang Li, Lei Ji, Fan Yang, Guillermo
  Sapiro, and Nan Duan.
\newblock {GODIVA}: Generating open-domain videos from natural descriptions.
\newblock {\em arXiv preprint arXiv:2104.14806}, 2021.

\bibitem{hong2022cogvideo}
Wenyi Hong, Ming Ding, Wendi Zheng, Xinghan Liu, and Jie Tang.
\newblock {CogVideo}: Large-scale pretraining for text-to-video generation via
  transformers.
\newblock In {\em International Conference on Learning Representations}, 2023.

\bibitem{kondratyuk2024videopoet}
Dan Kondratyuk, Lijun Yu, Xiuye Gu, Jos{\'e} Lezama, Jonathan Huang, Grant
  Schindler, Rachel Hornung, Vighnesh Birodkar, Jimmy Yan, Ming-Chang Chiu,
  et~al.
\newblock {VideoPoet}: A large language model for zero-shot video generation.
\newblock In {\em International Conference on Machine Learning}, 2024.

\bibitem{yan2021videogpt}
Wilson Yan, Yunzhi Zhang, Pieter Abbeel, and Aravind Srinivas.
\newblock {VideoGPT}: Video generation using {VQ-VAE} and transformers.
\newblock {\em arXiv preprint arXiv:2104.10157}, 2021.

\bibitem{chen2024diffusionforcing}
Boyuan Chen, Diego Mart{\'i}~Mons{\'o}, Yilun Du, Max Simchowitz, Russ Tedrake,
  and Vincent Sitzmann.
\newblock {Diffusion Forcing}: Next-token prediction meets full-sequence
  diffusion.
\newblock In {\em Advances in Neural Information Processing Systems},
  volume~37, 2024.

\bibitem{xie2025progressive}
Desai Xie, Zhan Xu, Yicong Hong, Hao Tan, Difan Liu, Feng Liu, Arie Kaufman,
  and Yang Zhou.
\newblock Progressive autoregressive video diffusion models.
\newblock In {\em Proceedings of the IEEE/CVF Conference on Computer Vision and
  Pattern Recognition Workshops}, pages 6376--6386, 2025.

\bibitem{sun2025ardiffusion}
Mingzhen Sun, Weining Wang, Gen Li, Jiawei Liu, Jiahui Sun, Wanquan Feng,
  Shanshan Lao, Siyu Zhou, Qian He, and Jing Liu.
\newblock {AR-Diffusion}: Asynchronous video generation with auto-regressive
  diffusion.
\newblock In {\em Proceedings of the IEEE/CVF Conference on Computer Vision and
  Pattern Recognition}, pages 7364--7373, 2025.

\bibitem{teng2025magi}
{Sand.ai}, Hansi Teng, Hongyu Jia, Lei Sun, Lingzhi Li, Maolin Li, Mingqiu
  Tang, Shuai Han, Tianning Zhang, W.~Q. Zhang, et~al.
\newblock {MAGI-1}: Autoregressive video generation at scale.
\newblock {\em arXiv preprint arXiv:2505.13211}, 2025.

\bibitem{kodaira2026streamdit}
Akio Kodaira, Tingbo Hou, Ji~Hou, Markos Georgopoulos, Felix Juefei-Xu,
  Masayoshi Tomizuka, and Yue Zhao.
\newblock {StreamDiT}: Real-time streaming text-to-video generation.
\newblock In {\em Proceedings of the IEEE/CVF Conference on Computer Vision and
  Pattern Recognition}, pages 29200--29210, 2026.

\bibitem{liu2025rolling}
Kunhao Liu, Wenbo Hu, Jiale Xu, Ying Shan, and Shijian Lu.
\newblock {Rolling Forcing}: Autoregressive long video diffusion in real time.
\newblock In {\em International Conference on Learning Representations}, 2026.

\bibitem{liu2026diagdistill}
Jinxiu Liu, Xuanming Liu, Kangfu Mei, Yandong Wen, Ming-Hsuan Yang, and Weiyang
  Liu.
\newblock Streaming autoregressive video generation via diagonal distillation.
\newblock In {\em International Conference on Learning Representations}, 2026.

\bibitem{qiu2024freenoise}
Haonan Qiu, Menghan Xia, Yong Zhang, Yingqing He, Xintao Wang, Ying Shan, and
  Ziwei Liu.
\newblock {FreeNoise}: Tuning-free longer video diffusion via noise
  rescheduling.
\newblock In {\em International Conference on Learning Representations}, 2024.

\bibitem{zhao2026relax}
Zengqun Zhao, Yanzuo Lu, Ziquan Liu, Jifei Song, Jiankang Deng, and Ioannis
  Patras.
\newblock {Relax Forcing}: Relaxed {KV}-memory for consistent long video
  generation.
\newblock In {\em British Machine Vision Conference}, 2026.
\newblock Accepted for publication.

\bibitem{kim2026memrope}
Youngrae Kim, Qixin Hu, C.-C.~Jay Kuo, and Peter~A. Beerel.
\newblock {MemRoPE}: Training-free infinite video generation via evolving
  memory tokens.
\newblock In {\em European Conference on Computer Vision}, 2026.

\bibitem{yang2025longlive}
Shuai Yang, Wei Huang, Ruihang Chu, Yicheng Xiao, Yuyang Zhao, Xianbang Wang,
  Muyang Li, Enze Xie, Yingcong Chen, Yao Lu, et~al.
\newblock {LongLive}: Real-time interactive long video generation.
\newblock In {\em International Conference on Learning Representations}, 2026.

\bibitem{chen2026contextforcing}
Shuo Chen, Cong Wei, Sun Sun, Tiancheng Shen, Ping Nie, Kai Zou, Ge~Zhang,
  Ming-Hsuan Yang, and Wenhu Chen.
\newblock {Context Forcing}: Consistent autoregressive video generation with
  long context.
\newblock In {\em International Conference on Machine Learning}, 2026.

\bibitem{zou2026hiar}
Kai Zou, Dian Zheng, Hongbo Liu, Tiankai Hang, Bin Liu, and Nenghai Yu.
\newblock {HiAR}: Efficient autoregressive long video generation via
  hierarchical denoising.
\newblock In {\em European Conference on Computer Vision}, volume 17008 of {\em
  Lecture Notes in Computer Science}, pages 225--242. Springer, 2026.

\bibitem{yuan2026helios}
Shenghai Yuan, Yuanyang Yin, Zongjian Li, Xinwei Huang, Xiao Yang, and Li~Yuan.
\newblock {Helios}: Real real-time long video generation model.
\newblock {\em arXiv preprint arXiv:2603.04379}, 2026.

\bibitem{alonso2024diamond}
Eloi Alonso, Adam Jelley, Vincent Micheli, Anssi Kanervisto, Amos Storkey, Tim
  Pearce, and Fran{\c{c}}ois Fleuret.
\newblock Diffusion for world modeling: Visual details matter in {Atari}.
\newblock In {\em Advances in Neural Information Processing Systems},
  volume~37, pages 58757--58791, 2024.

\bibitem{li2025gamecraft}
Jiaqi Li, Junshu Tang, Zhiyong Xu, Longhuang Wu, Yuan Zhou, Shuai Shao, Tianbao
  Yu, Zhiguo Cao, and Qinglin Lu.
\newblock {Hunyuan-GameCraft}: High-dynamic interactive game video generation
  with hybrid history condition.
\newblock {\em arXiv preprint arXiv:2506.17201}, 2025.

\bibitem{tencent2025hyworld15}
{Tencent Hunyuan}.
\newblock {HY-World 1.5}: A systematic framework for interactive world modeling
  with real-time latency and geometric consistency.
\newblock Technical report, Tencent Hunyuan, 2025.

\bibitem{shin2025motionstream}
Joonghyuk Shin, Zhengqi Li, Richard Zhang, Jun-Yan Zhu, Jaesik Park, Eli
  Shechtman, and Xun Huang.
\newblock {MotionStream}: Real-time video generation with interactive motion
  controls.
\newblock In {\em International Conference on Learning Representations}, 2026.

\bibitem{xu2026worldroambench}
Ting-Bing Xu, Jiacheng Sui, Zhe Gao, Kewei Shi, Wenjin Yang, Zhicheng Liu,
  Zhaoxu Sun, Mingchao Sun, Hongyu Pan, Fan Jiang, et~al.
\newblock {WorldRoamBench}: An open-world benchmark for long-horizon stability
  of interactive world models.
\newblock {\em arXiv preprint arXiv:2606.31672}, 2026.

\bibitem{lipman2022flow}
Yaron Lipman, Ricky T.~Q. Chen, Heli Ben-Hamu, Maximilian Nickel, and Matthew
  Le.
\newblock Flow matching for generative modeling.
\newblock In {\em International Conference on Learning Representations}, 2023.

\bibitem{yin2024dmd}
Tianwei Yin, Micha{\"e}l Gharbi, Richard Zhang, Eli Shechtman, Fr{\'e}do
  Durand, William~T. Freeman, and Taesung Park.
\newblock One-step diffusion with distribution matching distillation.
\newblock In {\em Proceedings of the IEEE/CVF Conference on Computer Vision and
  Pattern Recognition}, pages 6613--6623, 2024.

\bibitem{yin2024improved}
Tianwei Yin, Micha{\"e}l Gharbi, Taesung Park, Richard Zhang, Eli Shechtman,
  Fr{\'e}do Durand, and William~T. Freeman.
\newblock Improved distribution matching distillation for fast image synthesis.
\newblock In {\em Advances in Neural Information Processing Systems},
  volume~37, pages 47455--47487, 2024.

\bibitem{wan2025wan22repo}
{Team Wan}.
\newblock {Wan2.2}: Open and advanced large-scale video generative models.
\newblock GitHub repository, 2025.

\bibitem{wan2025wan}
{Team Wan}.
\newblock {Wan}: Open and advanced large-scale video generative models.
\newblock {\em arXiv preprint arXiv:2503.20314}, 2025.

\bibitem{huang2024vbench}
Ziqi Huang, Yinan He, Jiashuo Yu, Fan Zhang, Chenyang Si, Yuming Jiang, Yuanhan
  Zhang, Tianxing Wu, Qingyang Jin, Nattapol Chanpaisit, et~al.
\newblock {VBench}: Comprehensive benchmark suite for video generative models.
\newblock In {\em Proceedings of the IEEE/CVF Conference on Computer Vision and
  Pattern Recognition}, pages 21807--21818, 2024.

\bibitem{huang2025vbenchpp}
Ziqi Huang, Fan Zhang, Xiaojie Xu, Yinan He, Jiashuo Yu, Ziyue Dong, Qianli Ma,
  Nattapol Chanpaisit, Chenyang Si, Yuming Jiang, et~al.
\newblock {VBench++}: Comprehensive and versatile benchmark suite for video
  generative models.
\newblock {\em IEEE Transactions on Pattern Analysis and Machine Intelligence},
  2025.

\bibitem{alibaba2026happyoyster}
{Alibaba Cloud}.
\newblock Alibaba launches {HappyOyster}, a world model product for real-time
  immersive creation and interaction.
\newblock Alibaba Cloud Blog, 2026.

\bibitem{robbyant2026lingbotworld}
{Robbyant Team}, Zelin Gao, Qiuyu Wang, Yanhong Zeng, Jiapeng Zhu, Ka~Leong
  Cheng, Yixuan Li, Hanlin Wang, Yinghao Xu, Shuailei Ma, et~al.
\newblock Advancing open-source world models.
\newblock {\em arXiv preprint arXiv:2601.20540}, 2026.

\bibitem{gao2026lingbotworldv2}
Zelin Gao, Qiuyu Wang, Jiapeng Zhu, Jingye Chen, Zichen Liu, Qingyan Bai,
  Jiahao Wang, Yufeng Yuan, Hanlin Wang, Yichong Lu, et~al.
\newblock Infinite worlds with versatile interactions.
\newblock {\em arXiv preprint arXiv:2607.07534}, 2026.

\bibitem{jiang2026abotworld}
Fan Jiang, Zhaoxu Sun, Mengchao Wang, Ziyu Zhu, Chiyu Wang, Yunpeng Zhang,
  et~al.
\newblock {ABot-World-0}: Infinite interactive world rollout on a single
  desktop {GPU}.
\newblock {\em arXiv preprint arXiv:2607.19191}, 2026.

\bibitem{xue2025ultravideo}
Zhucun Xue, Jiangning Zhang, Teng Hu, Haoyang He, Yinan Chen, Yuxuan Cai,
  Yabiao Wang, Chengjie Wang, Yong Liu, Xiangtai Li, et~al.
\newblock {UltraVideo}: High-quality {UHD} video dataset with comprehensive
  captions.
\newblock In {\em Advances in Neural Information Processing Systems},
  volume~38, 2025.

\bibitem{roberts2004general}
Gareth~O. Roberts and Jeffrey~S. Rosenthal.
\newblock General state space {Markov} chains and {MCMC} algorithms.
\newblock {\em Probability Surveys}, 1:20--71, 2004.

\bibitem{hacohen2024ltx}
Yoav HaCohen, Nisan Chiprut, Benny Brazowski, Daniel Shalem, Dudu Moshe, Eitan
  Richardson, Eran Levin, Guy Shiran, Nir Zabari, Ori Gordon, et~al.
\newblock {LTX-Video}: Realtime video latent diffusion.
\newblock {\em arXiv preprint arXiv:2501.00103}, 2024.

\bibitem{deng2025nova}
Haoge Deng, Ting Pan, Haiwen Diao, Zhengxiong Luo, Yufeng Cui, Huchuan Lu,
  Shiguang Shan, Yonggang Qi, and Xinlong Wang.
\newblock Autoregressive video generation without vector quantization.
\newblock In {\em International Conference on Learning Representations}, 2025.

\bibitem{jin2024pyramidflow}
Yang Jin, Zhicheng Sun, Ningyuan Li, Kun Xu, Hao Jiang, Nan Zhuang, Quzhe
  Huang, Yang Song, Yadong Mu, et~al.
\newblock Pyramidal flow matching for efficient video generative modeling.
\newblock In {\em International Conference on Learning Representations}, 2025.

\bibitem{chen2025skyreelsv2}
Guibin Chen, Dixuan Lin, Jiangping Yang, Chunze Lin, Junchen Zhu, Mingyuan Fan,
  Hao Zhang, Sheng Chen, Zheng Chen, Chengcheng Ma, et~al.
\newblock {SkyReels-V2}: Infinite-length film generative model.
\newblock {\em arXiv preprint arXiv:2504.13074}, 2025.

\bibitem{ma2025hpsv3}
Yuhang Ma, Xiaoshi Wu, Keqiang Sun, and Hongsheng Li.
\newblock {HPSv3}: Towards wide-spectrum human preference score.
\newblock In {\em Proceedings of the IEEE/CVF International Conference on
  Computer Vision (ICCV)}, pages 15086--15095, 2025.

\end{thebibliography}
